\documentclass[letterpaper,twocolumn,10pt]{article}
\usepackage{usenix2019_v3}

\usepackage[skip=2pt,font=small]{caption}

\usepackage{cite}
\usepackage{amsmath,amssymb,amsfonts}
\usepackage{algorithm,algorithmicx,algpseudocode}
\usepackage{multirow}
\algnewcommand{\LeftCommenta}[1]{\Statex \hspace{1.3em} \(\triangleright\) #1}
\algnewcommand{\LeftCommentb}[1]{\Statex \hspace{2.7em} \(\triangleright\) #1}
\usepackage{graphicx}
\usepackage{textcomp}
\usepackage{xcolor}
\usepackage{booktabs}
\usepackage{hyperref}
\usepackage{bm}

\usepackage{diagbox}
\usepackage{amsthm}
\newtheorem{theorem}{Theorem}
\newtheorem{lemma}{Lemma}
\usepackage{array}
\usepackage{xspace}
\newcommand{\framework}{PatchGuard\xspace}

\begin{document}

\date{}

\title{PatchGuard: A Provably Robust Defense against Adversarial Patches via Small Receptive Fields and Masking}

\author{
{\rm Chong Xiang}\\
Princeton University
\and
{\rm Arjun Nitin Bhagoji}\\
University of Chicago
\and
{\rm Vikash Sehwag}\\
Princeton University
\and 
{\rm Prateek Mittal}\\
Princeton University
}

\maketitle

\begin{abstract}
Localized adversarial patches aim to induce misclassification in machine learning models by arbitrarily modifying pixels within a restricted region of an image. Such attacks can be realized in the physical world by attaching the adversarial patch to the object to be misclassified, and defending against such attacks is an unsolved/open problem. In this paper, we propose a general defense framework called \framework that can achieve high provable robustness while maintaining high clean accuracy against localized adversarial patches. The cornerstone of \framework involves the use of CNNs with small receptive fields to impose a bound on the number of features corrupted by an adversarial patch. Given a bounded number of corrupted features, the problem of designing an adversarial patch defense reduces to that of designing a secure feature aggregation mechanism. Towards this end, we present our \textit{robust masking} defense that robustly detects and masks corrupted features to recover the correct prediction. Notably, we can prove the robustness of our defense against any adversary within our threat model. Our extensive evaluation on ImageNet, ImageNette (a 10-class subset of ImageNet), and CIFAR-10 datasets demonstrates that our defense achieves state-of-the-art performance in terms of both provable robust accuracy and clean accuracy.\footnote{A shorter version of this paper is published at USENIX Security Symposium 2021. Our code is available at \texttt{\url{https://github.com/inspire-group/PatchGuard}} for the purpose of reproducibility.}

\end{abstract}
\section{Introduction}\label{sec-introduction}

Machine learning models are vulnerable to evasion attacks, where an adversary introduces a small perturbation to a test example for inducing model misclassification~\cite{szegedy2013intriguing,goodfellow2014explaining}. Many prior attacks and defenses focus on the classic setting of adversarial examples that have a small $L_p$ distance to the benign example~\cite{szegedy2013intriguing,goodfellow2014explaining,papernot2016limitations,carlini2017towards,madry2017towards,papernot2016distillation,xu2017feature,meng2017magnet,metzen2017detecting,carlini2017adversarial,athalye2018obfuscated,tramer2020adaptive}. However, in the physical world, the classic $L_p$ setting may require global perturbations to an object, which is not always practical. In this paper, we focus on the threat of \emph{localized adversarial patches}, in which the adversary can arbitrarily modify pixels within a small restricted area such that the perturbation can be realized by attaching an adversarial patch to the victim object. Several effective patch attacks have been shown: 1) Brown et al.~\cite{brown2017adversarial} generate physical adversarial patches that can force model predictions to be a target class of the attacker's choice; 2) Karmon et al.~\cite{karmon2018lavan} propose the LaVAN attack in the digital domain; 3) Eykholt et al.~\cite{evtimov2017robust} demonstrate a robust physical-world attack that attaches small stickers to a stop sign for fooling traffic sign recognition. 

The success of practical localized adversarial patches has inspired several defenses. Digital Watermark (DW)~\cite{hayes2018visible} aims to detect and remove the adversarial patch while Local Gradient Smoothing (LGS)~\cite{naseer2019local} proposes smoothing the suspicious region of pixels to neutralize the adversarial patch. However, these empirical defenses are heuristic approaches and lack robustness against a strong adaptive attacker~\cite{chiang2020certified}. This has led to the development of several {certifiably robust} defenses. Chiang et al.~\cite{chiang2020certified} propose the first certified defense against adversarial patches via Interval Bound Propagation (IBP)~\cite{gowal2018effectiveness,mirman2018differentiable}. Zhang et al.~\cite{zhang2020clipped} use a clipped BagNet (CBN) to achieve provable robustness while Levine et al.~\cite{levine2020randomized} propose De-randomized Smoothing (DS) to further improve provable robustness. These works have taken important steps towards provably robust models. However, their performance is still limited in terms of provable robustness and standard classification accuracy (i.e., clean accuracy), leaving defenses against adversarial patches an unsolved/open problem.

\begin{figure*}[!th]
    \centering
    \includegraphics[width=\textwidth]{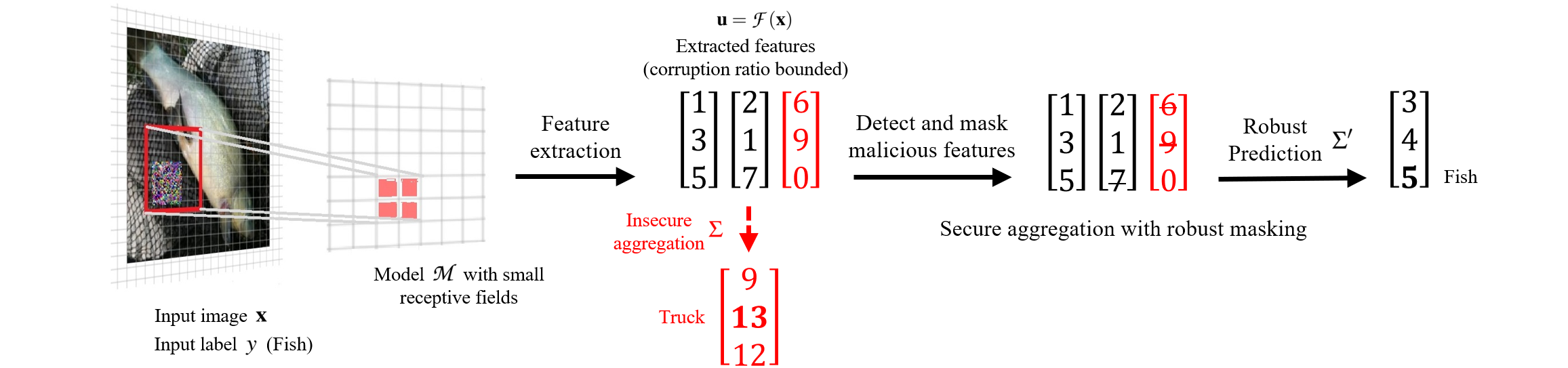}
    \caption{\small\textbf{Overview of defense}. The small receptive field bounds the number of corrupted features (one out of three vectors in this example). The one corrupted feature (red vector) in this example has an abnormally large element that dominates the insecure aggregation ($\Sigma$) but also leads to a distinct pattern from clean features. Our \emph{robust masking} aggregation detects and masks the corrupted feature, recovering the correct prediction from the remaining features. We note that \emph{robust masking} can have false positives (FP) and incorrectly mask benign features, but we show in Section~\ref{sec-evaluation} that our defense retains high clean accuracy and provable robust accuracy.}
    \label{fig-overview}
\end{figure*}

\subsection{Contributions}
In this paper, we propose a general defense framework called {\framework} that achieves \emph{substantial state-of-the-art provable robustness while maintaining high clean accuracy against localized adversarial patches}.

\textbf{Insight: Leverage CNNs with Small Receptive Fields.} The cornerstone of our defense framework involves the use of Convolutional Neural Networks (CNNs) with small receptive fields to impose a bound on the number of features that can be corrupted due to an adversarial patch. The receptive field of a CNN is the region of an input image that a particular feature is influenced by, and model prediction is based on the aggregation of features extracted from different regions of an image. An example of the receptive field is shown as the red box on the image in Figure~\ref{fig-overview}. Our case study in Section~\ref{sec-motivation} demonstrates that a large receptive field makes CNNs more vulnerable to adversarial patch attacks. For a model with a large receptive field of 483$\times$483 (ResNet-50~\cite{he2016deep}) on ImageNet images~\cite{deng2009imagenet}, a small patch is present in the receptive field of most extracted features and can thus easily change model prediction. A small receptive field, on the other hand, limits the number of corrupted features, and we use it as the fundamental building block of robust classifiers. We note that a small receptive field is not a barrier to achieving high clean accuracy. A ResNet-like architecture with a small 17$\times$17 receptive field can achieve an AlexNet-level accuracy for ImageNet top-5 classification~\cite{brendel2019approximating}. The potential robustness improvement, as well as the moderate accuracy drop, motivates the use of small receptive fields in \framework.

\textbf{Insight: Leveraging Secure Aggregation \& Robust Masking.} However, a small receptive field alone is not enough for robust prediction since conventional models use insecure feature aggregation mechanisms such as mean. The use of small receptive fields turns the problem of designing an adversarial patch defense into a secure aggregation problem, and we propose \emph{robust masking} as an effective instance of secure feature aggregation mechanism. Figure~\ref{fig-overview} provides an overview of our defense. The small receptive field ensures that only a small fraction of extracted features are corrupted due to an adversarial patch. The small number of corrupted features forces the adversary to create abnormally large feature values to dominate the final prediction, and \emph{robust masking} aims to detect and mask these abnormal features. Our empirical analysis  demonstrates that removing a small number of features of a clean image is unlikely to change model prediction. Therefore, robust masking recovers the correct prediction with high probability if all the corrupted features are masked.

\textbf{Provable Robustness.} Robust masking introduces a fundamental dilemma for the adversary: either to generate conspicuous malicious features that will be detected and masked by our defense or to do with stealthy but ineffective adversarial patches. In Section \ref{sec-provable}, we show that this dilemma leads to a proof of \emph{provable robustness} for our defense, providing the guarantee that the model can always recover correct predictions on certified images against any adversarial patch within the threat model. This is a stronger notion of robustness compared with defenses that only detect the adversarial attack~\cite{mccoyd2020minority,meng2017magnet,xu2017feature}. We also show that \framework subsumes several existing defenses~\cite{zhang2020clipped,levine2020randomized} (as shown in Section~\ref{sec-generalization-cbn-ds}), and outperforms them due to the use of \emph{robust masking}.

\textbf{State-of-the-art Performance.} We consider the strongest adversarial patch attacker, who can place the adversarial patch on any part of the image, including on top of salient objects. We evaluate our provably robust defense against any patch attacker on ImageNet~\cite{deng2009imagenet}, ImageNette~\cite{imagenette}, CIFAR-10~\cite{krizhevsky2009learning}, and shows that our defense achieves state-of-the-art performance in terms of provable robustness and clean accuracy compared to previous defenses~\cite{chiang2020certified,levine2020randomized,zhang2020clipped}.
Our main contributions can be summarized as follows:

\begin{enumerate}
    \item We demonstrate the use of a small receptive field as a fundamental building block for robustness and leverage it to develop our general defense framework called \framework. \framework is flexible and general as it is compatible with any CNN with small receptive fields and any secure aggregation mechanism. 
    \item  We present \emph{robust masking} as an instance of the secure aggregation mechanism that leads to provable robustness and recovers correct predictions for certified images against any attacker within the threat model. 
    \item We comprehensively evaluate our defense across ImageNet~\cite{deng2009imagenet}, ImageNette~\cite{imagenette}, CIFAR-10~\cite{krizhevsky2009learning} datasets, and demonstrate state-of-the-art provable robust accuracy and clean accuracy of our defense. 
\end{enumerate}

\section{Problem Formulation}\label{sec-formulation}
In this section, we first introduce the image classification model, followed by the adversarial patch attack and defense formulation. Finally, we present important terminology used in \framework. Table~\ref{tab-notation} provides a summary of our notation.

\subsection{Image Classification Model}\label{sec-model-formulation}
We focus on Fully Convolutional Neural Networks (FCNNs) such as ResNet~\cite{he2016deep}, which use convolutional layers for feature extraction and \textit{only one} additional fully-connected layer for the final classification. This structure is widely used in state-of-the-art image classification models~\cite{he2016deep,szegedy2015going,simonyan2014very,szegedy2017inception}.

We use $\mathcal{X} \subset [0,1]^{W\times H \times C}$ to denote the image space where each image has width $W$, height $H$, number of channels $C$, and the pixels are re-scaled to $[0,1]$. We take $\mathcal{Y}=\{0,1,\cdots,N-1\}$ as the label space, where the number of classes is $N$. 
We use $\mathcal{M}(\mathbf{x}):\mathcal{X} \rightarrow \mathcal{Y}$ to denote the model that takes an image $\mathbf{x}\in \mathcal{X}$ as input and predicts the class label $y\in\mathcal{Y}$. We let $\mathcal{F}(\mathbf{x}):\mathcal{X} \rightarrow \mathcal{U}$ be the feature extractor that outputs the feature tensor $\mathbf{u} \in \mathcal{U} \subset \mathbb{R}^{ W^\prime \times H^\prime \times C^\prime}$, where $W^\prime$, $H^\prime$, $C^\prime$ are the width, height, and number of channels in this feature map, respectively.

\subsection{Attack Formulation}\label{sec-attack}

\noindent \textbf{Attack objective.} We focus on evasion attacks against an image classification model. Given a deep learning model $\mathcal{M}$, an image $\mathbf{x}$, and its true class label $y$, the goal of the attacker is to find an image $\mathbf{x}^\prime \in \mathcal{A}(\mathbf{x}) \subset \mathcal{X}$ satisfying a constraint $\mathcal{A}$ such that $\mathcal{M}(\mathbf{x}^\prime) \neq y$. The constraint $\mathcal{A}$ is defined by the attacker's threat model, which we will describe below. We note that the attack objective of inducing misclassification into any wrong class is referred to as an \emph{untargeted attack}. In contrast, when the goal is to misclassify the image to a particular target class $y^\prime\neq y$, it is called a \textit{targeted attack}. The untargeted attack is easier to launch and thus more difficult to defend against. In this paper, we focus on defenses against the untargeted attack.

\noindent \textbf{Attacker capability.} The attacker can arbitrarily modify pixels within a restricted region, and this region can be anywhere on the image, even over the salient object. We assume that all manipulated pixels are within a contiguous region, and the defender has a conservative estimate (i.e., upper bound) of the region size. 
We note that this matches the strongest threat model used in the existing literature on certified defenses against adversarial patches \cite{chiang2020certified,levine2020randomized,zhang2020clipped}.\footnote{A high-performance provably robust defense against a single patch is currently an open/unsolved problem and is thus the focus of our threat model. We will discuss our defense extension for multiple patches in Appendix~\ref{apx-multiple-patch}.}
Formally, we use a binary \textit{pixel block} $\mathbf{p} \in P \subset \{0,1\}^{W\times H}$ to represent the restricted region, where the pixels within the region are set to $1$. Then, the constraint set $\mathcal{A}(\mathbf{x})$ can be expressed as $\{\mathbf{x}^\prime = (\mathbf{1}-\mathbf{p})\odot \mathbf{x} + \mathbf{p} \odot \mathbf{x}^{\prime\prime} | \mathbf{x},\mathbf{x}^\prime \in \mathcal{X}, \mathbf{x}^{\prime\prime}\in[0,1]^{W\times H \times C}, \mathbf{p} \in P\}$, where $\odot$ refers to the element-wise product operator, and $\mathbf{x}^{\prime\prime}$ is the content of the adversarial patch. In this paper, we primarily focus on the case where $\mathbf{p}$ represents one square region. Our defense can generalize to other shapes and we provide the discussion and experiment results in Appendix~\ref{apx-patch-shape}.

\subsection{Defense Formulation}

\noindent \textbf{Defense objective.} The goal of our defense is to design a defended model $\mathcal{D}$ such that $\mathcal{D}(\mathbf{x})=\mathcal{D}(\mathbf{x}^\prime)=y$ for any clean data point $(\mathbf{x},y)\in\mathcal{X}\times\mathcal{Y}$ and any adversarial example $\mathbf{x}^\prime \in \mathcal{A}(\mathbf{x})$, where $\mathcal{A}(\mathbf{x})$ is the adversarial constraint introduced in Section~\ref{sec-attack}. Note that we aim to \emph{recover the correct prediction}, which is harder than merely detecting an attack.

\noindent\textbf{Provable robustness.} Previous works~\cite{carlini2017adversarial,chiang2020certified,tramer2020adaptive} have shown that empirical defenses are usually vulnerable to an adaptive white-box attacker who has full knowledge of the defense algorithm, model architecture, and model weights; therefore, we design \framework as a provably robust defense~\cite{gowal2018effectiveness,mirman2018differentiable,cohen2019certified,chiang2020certified,levine2020randomized,zhang2020clipped} to provide the strongest robustness. \emph{The evaluation of provable defense is agnostic to attack algorithms and its result holds for any attack considered in the threat model}.

\begin{table}[t]
    \centering
    \caption{Table of notation}

   \resizebox{\linewidth}{!}{ \begin{tabular}{l|l}
    \toprule
    \textbf{Notation} & \textbf{Description} \\
    \midrule
    $\mathcal{X} \subset[0,1]^{W\times H \times C}$  & Image space\\
    $\mathcal{Y} = \{0,1,\cdots,N-1\}$ & Label space\\
    $\mathcal{U} \subset \mathbb{R}^{ W^\prime \times H^\prime \times C^\prime}$ & Feature space\\
      $\mathcal{M}(\mathbf{x}):\mathcal{X} \rightarrow \mathcal{Y}$  & Model predictor from $\mathbf{x}\in\mathcal{X}$\\
      $\mathcal{F}(\mathbf{x}):\mathcal{X} \rightarrow \mathcal{U}$ & Local feature extractor for all classes \\
      $\mathcal{F}(\mathbf{x},l):\mathcal{X}\times \mathcal{Y} \rightarrow \mathcal{U}$ & Local feature extractor for class $l$\\
      $P \subset \{0,1\}^{W\times H }$ & Set of binary pixel blocks in the image space\\
      $W \subset \{0,1\}^{W^\prime\times H^\prime }$ & Set of binary windows in the feature space\\
      \bottomrule
    \end{tabular}}
    \label{tab-notation}
\end{table}

\subsection{\framework Terminology}\label{sec-defense-notation}

\noindent\textbf{Local feature and its receptive field.} Recall that we use $\mathcal{F}$ to extract feature map as $\mathbf{u}\in \mathbb{R}^{ W^\prime \times H^\prime \times C^\prime}$. We refer to each $1\times1\times C^\prime$-dimensional feature in tensor $\mathbf{u}$ as a \emph{local feature} since it is only extracted from part of the input image as opposed to the entire image. We define the \textit{receptive field} of a local feature to be a subset of image pixels that the feature $\tilde{\mathbf{u}} \in \mathbb{R}^{1\times1\times C^\prime}$ is looking at, or affected by. Formally, if we represent the input image ${x}$ as \emph{a set of pixels}, the receptive field of a particular local feature $\tilde{\mathbf{u}}$ is a subset of pixels for which the gradient of $\tilde{\mathbf{u}}$ is non-zero, i.e., $\{{r}\in {x} | \nabla_{{r}} \tilde{\mathbf{u}} \neq \mathbf{0} \}$. For simplicity, we use the phrase ``receptive field of a CNN" to refer to ``receptive field of a particular feature of a CNN".

\noindent\textbf{Global feature and global logits.} When the local feature tensor $\mathbf{u}$ is the output of the last convolutional layer, conventional CNNs use an element-wise linear aggregation (e.g., mean) over all local features to obtain the \emph{global feature} in $\mathbb{R}^{C^\prime}$. The global feature will then go through the last fully-connected layer (i.e., classification layer) and yield the \emph{global logits} vector in $\mathbb{R}^{N}$ for the final prediction (top of Figure~\ref{fig-feature}).

\noindent \textbf{Local logits.} Similar to computing the global logits from the global feature, we can feed each local feature (in $\mathbb{R}^{1\times1\times C^\prime}$) to the fully-connected layer to get the \emph{local logits} (in $\mathbb{R}^{1\times1\times N}$). Each local logits vector is the classification output based on each {local feature}; thus, they share the same receptive field. Concatenating all $W^\prime \cdot H^\prime$ local logits vectors gives the local logits tensor, and applying the element-wise linear aggregation gives the same global logits (bottom of Figure~\ref{fig-feature}).

\noindent \textbf{Local confidence, local prediction, and class evidence.} Based on local logits, we can derive the concept of \emph{local confidence} and \emph{local prediction} tensor by feeding the {local logits} tensor to a softmax layer and an argmax layer, respectively. In the remainder of this paper, we \emph{specialize the concept of feature} by considering it to refer to either a logits tensor, a confidence tensor, or a prediction tensor. In this case, we have $C^\prime=N$. We also sometimes abuse the notation by letting $\mathcal{F}(\mathbf{x},l):\mathcal{X}\times \mathcal{Y} \rightarrow \mathbb{R}^{ W^\prime \times H^\prime}$ denote the slice of the feature corresponding to class $l$. We call the elements of $\mathcal{F}(\mathbf{x},l)$ the \emph{class evidence} for class $l$.

\begin{figure}[t]
    \centering
    \includegraphics[width=0.95\linewidth]{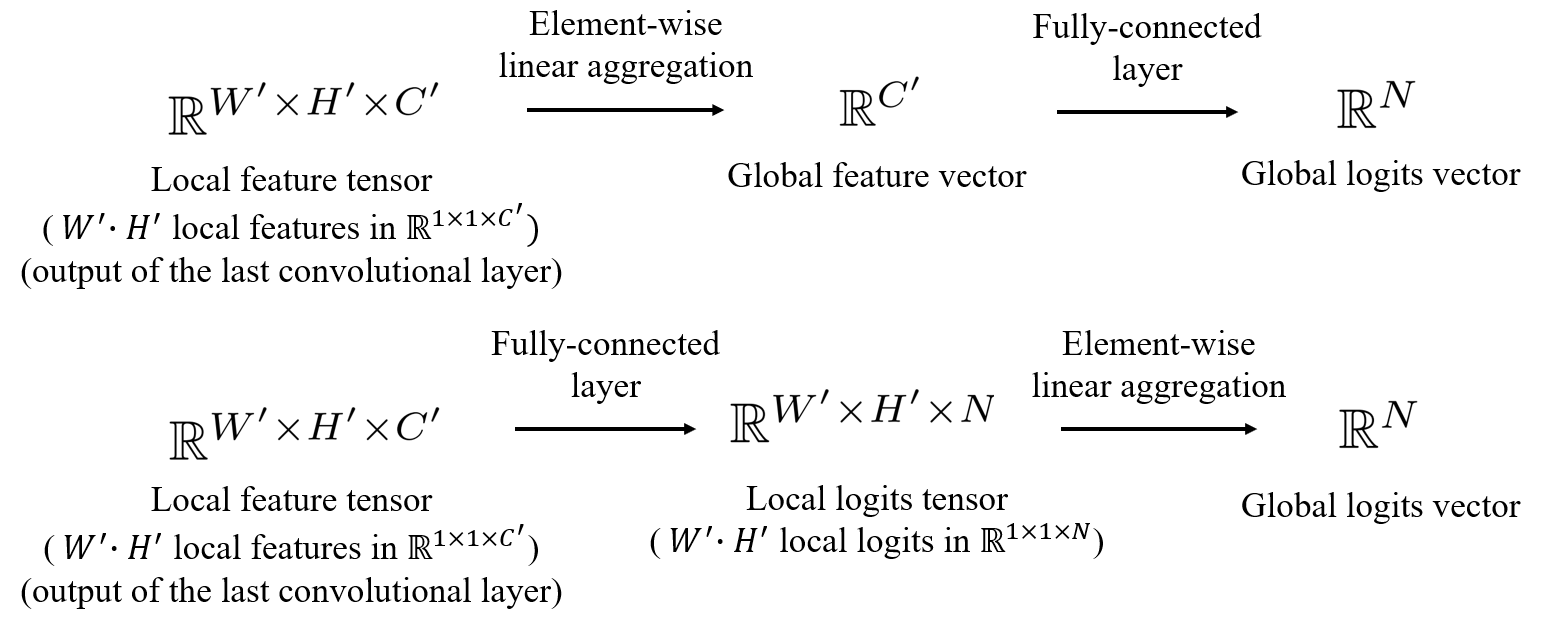}
    \caption{\small Two equivalent ways of computing the global logits vector (top: used in conventional CNNs; bottom: used in our defense).} 
    \label{fig-feature}
\end{figure}

\section{\framework}\label{sec-defense}

In this section, we first use an empirical case study to motivate the use of small receptive fields and secure feature aggregation (i.e., \emph{robust masking}).
Next, we will give an overview of our general \framework framework, followed by our use of networks with small receptive fields and details of our \emph{robust masking} based secure aggregation. The provable robustness of this defense will be demonstrated and analyzed in Section~\ref{sec-provable}.

\subsection{Why are adversarial patches effective?}\label{sec-motivation}

Previous work~\cite{brown2017adversarial,karmon2018lavan} on adversarial patches, surprisingly, shows that model prediction can be manipulated by patches that occupy a very small portion of input images. In this subsection, we provide a case study for ResNet-50~\cite{he2016deep} trained on ImageNet~\cite{deng2009imagenet}, ImageNette (a 10-class subset of ImageNet)~\cite{imagenette}, and CIFAR-10~\cite{krizhevsky2009learning} datasets and identify two critical reasons for the model vulnerability. These will then motivate the development and discussion of our defense.

\noindent\textbf{Experiment setup.} We take 5000 random ImageNet validation images and the entire validation sets of ImageNette and CIFAR-10 for the case study. We use a patch consisting of 3\% of the image pixels for an empirical attack.
Further details about the attack setup and datasets are covered in Appendix~\ref{apx-attack} and \ref{apx-setup}. We extract the local logits (as defined in Section~\ref{sec-defense-notation}) from adversarial images for further analysis.

\begin{table}[t]
    \centering
    \caption{Percentage of incorrect predictions of ResNet-50}
    \label{tab-local-attack}
   \resizebox{\linewidth}{!}{   \begin{tabular}{c|c|c|c}
    \toprule
    Dataset     &  ImageNet & ImageNette & CIFAR-10\\
    Patch size  & 3\% pixels & 3\% pixels & 3\% pixels\\
         \midrule
Incorrect local pred. (attacked) &84.4\% &56.4\%&67.0\%\\
Incorrect local pred. (original)& 59.9\% & 15.3\% & 27.0\%\\
Incorrect local pred. (difference) &24.5\% &41.1\% & 40.0\% \\
Incorrect global predictions &99.9\%&99.1\%&95.5\%\\
\bottomrule
    \end{tabular}}
\end{table}
\begin{figure}[t]
    \centering
    \includegraphics[width=0.7\linewidth]{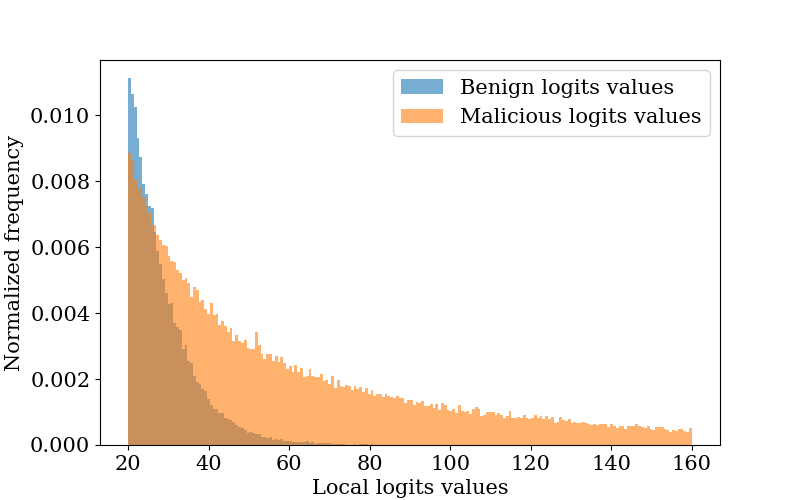}
    \caption{Histogram of large local logits values for ImageNet adversarial images (only positive values larger than 20 are shown; similar results for ImageNette and CIFAR-10 are in Appendix~\ref{apx-local-attack}).}
    \label{fig-hist-net}
\end{figure}

\noindent\textbf{Vulnerability I: the small adversarial patch appears in the large receptive fields of most local features and is able to manipulate the local predictions.}
In Table~\ref{tab-local-attack}, we report the percentage of incorrect local predictions of the adversarial images (attacked) and clean images (original) as well as their percentage difference. We can see that a small patch that only takes up 3\% of the image pixels can corrupt 24.5\% additional local predictions for ImageNet images, 41.1\% for ImageNette, and 40.0\% for CIFAR-10. As shown in the table, the large portion of incorrect local predictions finally leads to a high percentage of incorrect global predictions.
This vulnerability mainly stems from the large receptive field of ResNet-50. Each local feature of ResNet-50 is influenced by a 483$\times$483 pixel region in the input space (with zero padding)~\cite{araujo2019computing}; therefore, even if the adversarial patch only appears in a small restricted area, it is still within the receptive field of many local features and can manipulate the local predictions.\footnote{We note that a patch appearing in the receptive field of a local feature does not necessarily indicate a successful local feature corruption. Each local feature focuses exponentially more on the center of its receptive field (further details are in Appendix~\ref{apx-receptive}). When the adversarial patch is far away from the center of the receptive field, its influence on the feature is greatly limited.} This observation motivates the use of small receptive fields: if the receptive field is small, it ensures that only a limited number of local features can be corrupted by an adversarial patch, and robust prediction may be possible.

\noindent\textbf{Vulnerability II: the adversarial patch creates large malicious local feature values and makes linear feature aggregation insecure.} In Figure~\ref{fig-hist-net}, we plot the histogram of class evidence of the true class and the malicious class of the adversarial images from ImageNet.
As we can see from Figure~\ref{fig-hist-net}, the adversarial patch tends to create extremely large malicious class evidence to increase the chance of a successful attack. Conventional CNNs use simple linear operations such as average pooling to aggregate all local features, and thus are vulnerable to these large malicious feature values. This observation motivates our development of \textit{robust masking} as a secure feature aggregation mechanism.

\subsection{Overview of \framework}
\label{sec-overview}
In Section~\ref{sec-motivation}, we identified the large receptive field and insecure aggregation of conventional CNNs as two major sources of model vulnerability. In this subsection, we provide an overview of our defense that tackles both problems.

Recall that Figure~\ref{fig-overview} provides an overview of our defense framework. We consider a CNN $\mathcal{M}$ with small receptive fields. The feature extractor $\mathcal{F}(\mathbf{x})$ produces the local feature tensor $\mathbf{u}$ extracted from the input image $\mathbf{x}$, where $\mathbf{u}$ can be any one of the logits, confidence, or model prediction tensor. Our defense framework is compatible with any CNN with small receptive fields, and we will present two general ways of building such networks in Section~\ref{sec-small-field}. The small receptive field ensures that only a small fraction of features are corrupted by a localized adversarial patch. However, the insecure aggregation of these features via average pooling or summation might still result in a misclassification. To address this vulnerability, we propose a \emph{robust masking} algorithm for secure feature aggregation. 

In \textit{robust masking}, we detect and mask the corrupted features in the local feature tensor $\mathbf{u}=\mathcal{F}(\mathbf{x})$. Since the number of corrupted local features is limited due to the small receptive field, the adversary is forced to create large feature values to dominate the global prediction. These large feature values lead to a distinct pattern and enable our detection of corrupted features. Further, we empirically find that that model predictions are generally invariant to the removal of partial features (Section~\ref{sec-eval-detail-vanilla}). Therefore, once the corrupted features are masked, we are likely to recover the correct prediction $y$ with the remaining local features (right part of Figure~\ref{fig-overview}). This defense introduces a dilemma for the adversary: either to generate conspicuous malicious features that will be detected and masked by our defense or to use stealthy but ineffective adversarial patches. This fundamental dilemma enables \emph{provable robustness}. We will introduce the details of robust masking in Section~\ref{sec-masking}, and perform its provable analysis in Section~\ref{sec-provable}.

\begin{figure}[t]
    \centering
    \includegraphics[width=\linewidth]{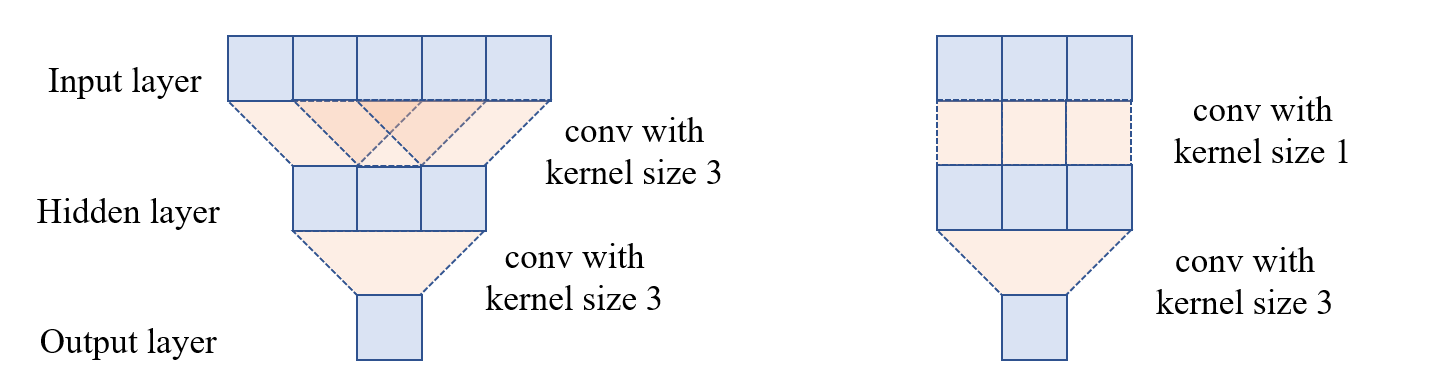}
    \caption{Effect of the convolution kernel size on the output receptive field size (left: two convolutions with a kernel size of 3; right: two convolutions with a kernel size of 1 and 3, respectively).}
    \label{fig-receptive}
\end{figure}

\subsection{CNNs with Small Receptive Fields}\label{sec-small-field}

Our defense framework is \emph{compatible with any CNN with small receptive fields}.\footnote{The receptive field should be \textit{small} compared with the input image size.} In this subsection, we discuss two general ways to build such CNNs; our goal is to reduce the number of image pixels that can affect a particular feature.

\noindent \textbf{Building an ensemble model.} One approach to design a network with small receptive fields is to divide the original image into multiple small pixel patches and feed each pixel patch to a \emph{base} model for separate classification. We can then build an \emph{ensemble} model aggregating the output of \emph{base} models. In this ensemble model, a local feature is the base model output, which can be logits, confidence, or prediction. Since the base model only takes a small pixel patch as input, each local feature is only affected by a small number of pixels, and thus the ensemble model has a small receptive field. We note that as the image resolution becomes higher, the number of all possible pixel patches increases greatly, which leads to a huge training and testing computation cost of the ensemble model. A natural approach to reduce the computation cost is to do inference on a sub-sampled set of small pixel patches.

\noindent \textbf{Using small convolution kernels.} A more efficient approach is to use small convolution kernels in conventional CNN architectures. 
In Figure~\ref{fig-receptive}, we provide an illustration for 1-D convolution computation with different kernel sizes. As we can see, the output cell is affected by all 5 input cells when using two convolutions with a kernel size of 3 (left) while each output cell is only affected by 3 input cells when reducing the size of one kernel to 1 (right). This logic extends directly to the large CNNs used in practice by replacing large convolution kernels with small kernels. 
Moreover, we can use a convolution stride to skip a portion of small pixel patches to reduce the computation cost. The modified CNN can be regarded as an ensemble model from a subset of all possible pixel patches. With this formulation, we can efficiently extract all local features with \emph{one-time} model feed-forward computation. 
In Section~\ref{sec-evaluation}, we will instantiate both approaches by adapting the implementation from Levine et al.~\cite{ds} and Brendel et al.~\cite{bagnet} and compare their performance in terms of accuracy and efficiency.

\noindent \textbf{Remark: translation from images into features.} \textit{The use of CNNs with small receptive fields translates the adversarial patch defense problem from the image space to the feature space.} That is, the problem becomes one of performing robust prediction from the feature space where a limited-size contiguous region is corrupted (due to a limited-size contiguous adversarial patch in the image space). The security analysis in the feature space (i.e., local logits, confidence, or prediction tensor) is simplified due to the use of linear aggregation, in contrast with the high non-linearity of CNN models if we directly analyze the input image. This observation enables our efficient robust masking technique as well as our provable analysis.

\subsection{Robust Masking}\label{sec-masking}
Given that an adversarial patch can only corrupt a limited number of local features with small receptive fields, the adversary is forced to create a small region of abnormally high feature values to induce misclassification. In order to detect this corrupted region, we clip the feature values and use a sliding window to find the region with the highest class evidence for each of the classes. We then apply a mask to the suspected region for each class so that the final classification is not influenced by the adversarial features. The defense algorithm is shown in Algorithm~\ref{alg-masking}. 

\noindent\textbf{Clipping.} As shown in Algorithm~\ref{alg-masking}, our defense will iterate over all possible classes in $\mathcal{Y}$. For each class $\bar{y}$, we first get its corresponding clipped local feature tensor $\hat{\mathbf{u}}_{\bar{y}}$ from the undefended model. We set the default values of the clipping bounds to $c_l=0,c_h=\infty$ for all feature types and datasets. When the feature type is logits, we clip the negative values to zero since our empirical analysis in Section~\ref{sec-eval-detail-vanilla} shows that they contribute little to the correct prediction of clean images but can be abused by the adversary to reduce the class evidence of the true class. If the feature is a confidence tensor or one-hot encoded prediction, it is unaffected by clipping, since its values are already bounded in $[0,1]$. 

\noindent\textbf{Feature windows.} We use a sliding window to detect and mask the abnormal region in the feature space. A window is a binary mask in the feature space whose size matches the upper bound of the number of local features that can be corrupted by the adversarial patch. Formally, let \texttt{p} be the upper bound of patch size in the threat model, \texttt{r} be the size of receptive field, and \texttt{s} be the stride of receptive field, which is the pixel distance between two adjacent receptive centers. We can compute the optimal window size \texttt{w} as
\begin{equation}\label{eq-receptive}
    \texttt{w} = \lceil(\texttt{p}+\texttt{r}-1)/\texttt{s}\rceil
\end{equation}
This equation can be derived by considering the worst-case patch location and counting the maximum number of corrupted local features. A detailed derivation is in Appendix~\ref{apx-receptive}. We note that the window size is a tunable security parameter and we use a conservative window size (computed with the upper bound of the patch size) to make robust masking agnostic to the actual patch size used in an attack. The implications of using an \emph{overly} conservative window size are discussed in Section~\ref{sec-eval-detail-masking} and Appendix~\ref{apx-mismatch}. We represent each window $\mathbf{w}$ with a binary feature map in $\{0,1\}^{W^\prime\times H^\prime}$, where features within the window have values of one. 

\begin{algorithm}[t]
\caption{Robust masking}\label{alg-masking}
\begin{algorithmic}[1]
\renewcommand{\algorithmicrequire}{\textbf{Input:}}
\renewcommand{\algorithmicensure}{\textbf{Output:}}
\Require Image $\mathbf{x}$, label space $\mathcal{Y}$, feature extractor $\mathcal{F}$ of model $\mathcal{M}$, clipping bound $[c_l,c_h]$, the set of sliding windows $\mathcal{W}$, and detection threshold $T\in[0,1]$. Default setting: $c_l=0,c_h=\infty,T=0$.
\Ensure  Robust prediction $y^*$
\Procedure{RobustMasking}{}
\For{each $\bar{y} \in \mathcal{Y}$}
\State $\mathbf{u}_{\bar{y}} \gets \mathcal{F}(\mathbf{x},\bar{y})$ \Comment{Local feature for class $\bar{y}$} 
\State $\hat{\mathbf{u}}_{\bar{y}} \gets \textsc{Clip}(\mathbf{u}_{\bar{y}},c_l,c_h)$ \Comment{Clipped local features}
\State $\mathbf{w}^*_{\bar{y}} \gets \textsc{Detect}(\hat{\mathbf{u}}_{\bar{y}},T,\mathcal{W})$ \Comment{Detected window}

\State $s_{\bar{y}} \gets \textsc{Sum}(\hat{\mathbf{u}}_{\bar{y}}\odot (\mathbf{1}-\mathbf{w}^*_{\bar{y}}))$ \Comment{Applying the mask} 

\EndFor

\State $y^*\gets \arg\max_{\bar{y}\in\mathcal{Y}}(s_{\bar{y}})$

\State\Return  $y^*$ 
\EndProcedure

\item[]
\Procedure{Detect}{$\hat{\mathbf{u}}_{\bar{y}},T,\mathcal{W}$}
\State $\mathbf{w}^*_{\bar{y}} \gets \arg\max_{\mathbf{w} \in \mathcal{W}} \textsc{Sum}(\mathbf{w}\odot \hat{\mathbf{u}}_{\bar{y}})$ \Comment{Detection}
\State $b\gets \textsc{Sum}(\mathbf{w}^*_{\bar{y}}\odot \hat{\mathbf{u}}_{\bar{y}}) / \textsc{Sum}(\hat{\mathbf{u}}_{\bar{y}})$ \Comment{Normalization}
\If {$b \leq T$}
\State $\mathbf{w}^*_{\bar{y}} \gets \mathbf{0}$ \Comment{An empty mask returned}
\EndIf
\State \Return $\mathbf{w}^*_{\bar{y}}$
\EndProcedure

\end{algorithmic} 
\end{algorithm}

\noindent\textbf{Detection.} We use the subprocedure \textsc{Detect} to examine the clipped local feature tensor $\hat{\mathbf{u}}_{\bar{y}}$ and detect the suspicious region. \textsc{Detect} takes the feature tensor $\hat{\mathbf{u}}_{\bar{y}}$, the normalized detection threshold $T\in[0,1]$, and a set of sliding windows $\mathcal{W}$ as inputs. To detect the malicious region, \textsc{Detect} calculates the sum of feature values (i.e., the class evidence) for class $\bar{y}$ within every possible window and identifies the window with the highest sum of class evidence. If the normalized highest class evidence exceeds the threshold $T$, we return the corresponding window $\mathbf{w}^*_{\bar{y}}$ as the suspicious window for that class; otherwise, we return an empty window $\mathbf{0}$.

\noindent\textbf{Masking.} If we detect a suspicious window in the local feature space, we mask the features within the suspicious area and calculate the sum of class evidence from the remaining features as $s_{\bar{y}} = \textsc{Sum}(\hat{\mathbf{u}}_{\bar{y}}\odot (\mathbf{1}-\mathbf{w}^*_{\bar{y}}))$. After we calculate the masked class evidence $s_{\bar{y}}$ for all possible classes in $\mathcal{Y}$, the defense outputs the prediction as the class with largest class evidence, i.e., $y^*=\arg\max_{\bar{y}\in\mathcal{Y}}(s_{\bar{y}})$.

\section{Provable Robustness Analysis}\label{sec-provable}

In this section, we provide provable robustness analysis for our robust masking defense. For any clean image $\mathbf{x}$ and a given model $\mathcal{M}$, we will determine whether \textit{any attacker}, \emph{with the knowledge of our defense}, can bypass the robust masking defense. Recall that our threat model allows the adversarial patches to be within one restricted region.
Given this threat model, all the corrupted features will also be within a small window in the feature map space when using a CNN with small receptive fields; we call this window \emph{malicious window}. 

\noindent \textbf{Provable Robustness via an adversary dilemma.} With the robust masking defense, we put the adversary in a dilemma. If the adversary wants to succeed in the attack, they need to increase the class evidence of a wrong class. However, increasing the class evidence will trigger our detection and masking mechanism that reduces the class evidence. As a result, this dilemma imposes an upper bound on the class evidence of any class ($s_{\bar{y}}$ in Line 6 of Algorithm~\ref{alg-masking}), which further enables provable robustness. In fact, we can first prove the following lemma. 

\begin{lemma}\label{lemma}
Given a malicious window $\mathbf{w} \in \mathcal{W}$, a class $\bar{y} \in \mathcal{Y}$, the set of sliding windows $\mathcal{W}$, the clipped and masked class evidence of class $\bar{y}$ (i.e., $s_{\bar{y}}$ in Algorithm~\ref{alg-masking}) can be no larger than $\textsc{Sum}(\hat{\mathbf{u}}_{\bar{y}}\odot (\mathbf{1}-\mathbf{w}))/(1-T)$ when setting $c_l=0$ and $T\in[0,1)$.
\end{lemma}

\begin{proof}
The goal of the adversary is to modify the content within the malicious window $\mathbf{w}$ to bypass our defense. Let $e$ be the amount of class evidence within $\mathbf{w}$ and $t=\textsc{Sum}(\hat{\mathbf{u}}_{\bar{y}}\odot (\mathbf{1}-\mathbf{w}))$ be the class evidence outside $\mathbf{w}$. Note that the adversary has control over the value $e$ but not $t$, and that the total class evidence of the modified malicious feature tensor is now $t+e$. Next, the subprocedure \textsc{Detect} will take the malicious feature tensor as input and detect a suspicious window $\mathbf{w}^*_{\bar{y}}$. Finally, a mask is applied and the class evidence is reduced to $s_{\bar{y}}=t+e-e^\prime$, where $e^\prime$ is the class evidence within the detected window $\mathbf{w}^*_{\bar{y}}$. To obtain the upper bound of $s_{\bar{y}}$ given a specific malicious window $\mathbf{w}$, we will determine the ranges of $e,e^\prime$ in four possible cases of the detected window $\mathbf{w}^*_{\bar{y}}$, as illustrated in Figure~\ref{fig-lemma}.
\begin{enumerate}
    \item \textit{Case I: the malicious window is perfectly detected.} In this case, we have $\mathbf{w}=\mathbf{w}^*_{\bar{y}}$ and thus $e=e^\prime$. The class evidence $s_{\bar{y}}=t+e-e^\prime=t$. 
    \item \textit{Case II: a benign window is incorrectly detected.} In this case, we have $e^\prime=\textsc{Sum}(\hat{\mathbf{u}}_{\bar{y}}\odot\mathbf{w}^*_{\bar{y}})$. The adversary has the constraint that $e\leq e^\prime$; otherwise, the malicious window $\mathbf{w}$ instead of $\mathbf{w}^*_{\bar{y}}$ will be detected. Therefore, we have $s_{\bar{y}}=t+e-e^\prime \leq t$.
    \item \textit{Case III: the malicious window is partially detected.} Let $\mathbf{r}_1 = \mathbf{w}^*_{\bar{y}}\odot (\mathbf{1}-\mathbf{w})$ be the detected benign region, $\mathbf{r}_2=\mathbf{w}^*_{\bar{y}}\odot \mathbf{w}$ be the detected malicious region, and $\mathbf{r}_3=(\mathbf{1}-\mathbf{w}^*_{\bar{y}})\odot \mathbf{w}$ be the undetected malicious region. Let $q_1,q_2,q_3$ be the class evidence within region $\mathbf{r}_1,\mathbf{r}_2,\mathbf{r}_3$, respectively. We have $e=q_2+q_3$ and $e^\prime=q_1+q_2$. Similar to \textit{Case II}, the adversary has the constraint that $e \leq e^\prime$, or $q_3 \leq q_1$; otherwise, $\mathbf{w}$ instead of $\mathbf{w}^*_{\bar{y}}$ will be detected. Therefore, we have $s_{\bar{y}}=t+e-e^\prime = t + q_3 -q_1 \leq t$.
    \item \textit{Case IV: no suspicious window detected.} This case happens when the largest sum within every possible window does not exceed the detection threshold. We have $e/{(e+t)}\leq T$, which yields $e\leq tT/(1-T)$. We also have $e^\prime=0$ since no mask is applied. Therefore, the class evidence satisfies $s_{\bar{y}} = t+e \leq t/(1-T)$, where $T\in[0,1]$.
\end{enumerate}
Combining the above four cases, we have the upper bound of the target class evidence to be $t/(1-T)=\textsc{Sum}(\hat{\mathbf{u}}_{\bar{y}}\odot (\mathbf{1}-\mathbf{w}))/(1-T)$.
\end{proof}

\begin{figure}[t]
    \centering
    \includegraphics[width=\linewidth]{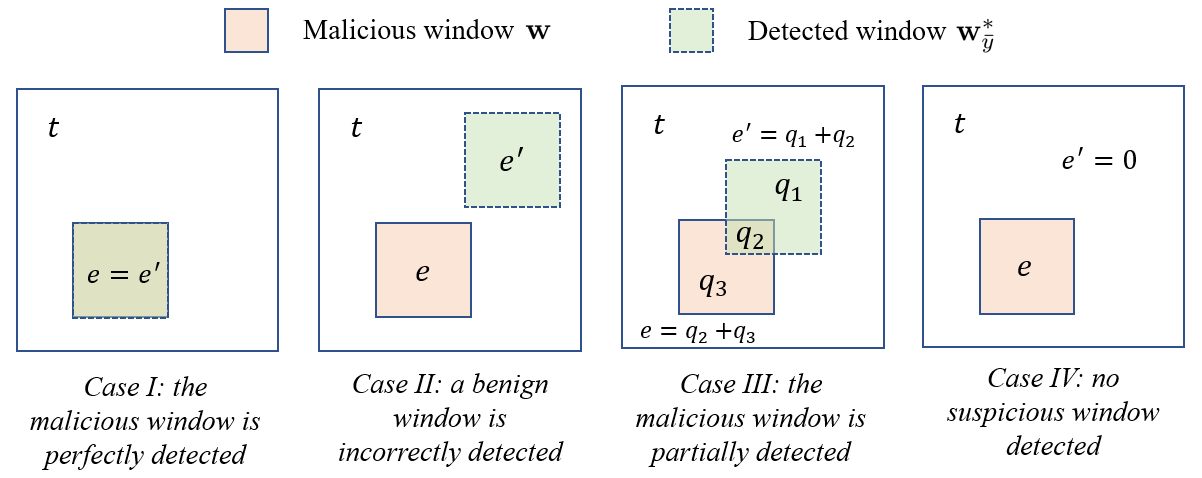}
    \caption{\small Illustrations for four cases of detected window $\mathbf{w}^*_{\bar{y}}$. The clipped and masked class evidence satisfies $s_{\bar{y}}=t+e-e^\prime$. For \textit{Case I, II, III}, we have $e\leq e^\prime$ and therefore $s_{\bar{y}}\leq t$. For \textit{Case IV}, we have $e\leq tT/(1-T),e^\prime=0$ and therefore $s_{\bar{y}}\leq t/(1-T)$.}
    \label{fig-lemma}
\end{figure}

\noindent \textbf{Provable analysis.} Lemma~\ref{lemma} shows that \textit{robust masking} limits the adversary's ability to increase the malicious class evidence. If the upper bound of malicious class evidence is not large enough to dominate the lower bound of the true class evidence, we can certify the robustness of our defense on a given clean image. The pseudocode of our provable analysis is provided in Algorithm~\ref{alg-provable-masking}. Next, we will explain our analysis by proving the following theorem.

\begin{theorem}\label{thm}
Let $c_l=0$, $T\in[0,1)$,  $\mathbf{w}\in\mathcal{W}$ denote the sliding windows whose sizes are determined by Equation~\ref{eq-receptive}, and $\mathcal{A}(\mathbf{x})$ denote the adversary's constraint as defined in Section~\ref{sec-attack}. 
If Algorithm~\ref{alg-provable-masking} returns \texttt{True} for a given image $\mathbf{x}$, our defense in Algorithm~\ref{alg-masking} can always make a correct prediction on any adversarial image $\mathbf{x}^\prime \in \mathcal{A}(\mathbf{x})$.
\end{theorem}

\begin{proof}
Our provable analysis in Algorithm~\ref{alg-provable-masking} iterates over all possible windows $\mathbf{w}\in\mathcal{W} $ and all possible target classes $y^\prime\in\mathcal{Y}^\prime=\mathcal{Y}\setminus\{y\}$ to derive provable robustness for the untargeted attack with a patch at any location. For each possible malicious window $\mathbf{w}$, Algorithm~\ref{alg-provable-masking} determines the upper bound of the class evidence of each target class (Line 3-6) and the lower bound of the class evidence of the true class (Line 7-9). 

For each target class $y^\prime$, we can apply Lemma~\ref{lemma} and get the upper bound $\overline{s}_{y^\prime}= \textsc{Sum}(\hat{\mathbf{u}}_{y^\prime}\odot(\mathbf{1}-\mathbf{w}))/(1-T)$. 

For the true class $y$, the optimal attacking strategy is to set all true class evidence within the malicious window $\mathbf{w}$ to $c_l=0$. Note that the true class evidence within the detected window $\mathbf{w}^*_y$ (if any) will be masked. Therefore, the lower bound $\underbar{s}_y$ is equivalent to removing class evidence within $\mathbf{w}$ and $\mathbf{w}^*_y$, i.e., $\underbar{s}_y= \textsc{Sum}(\hat{\mathbf{u}}_{y}\odot(\mathbf{1}-\mathbf{w})\odot(\mathbf{1}-\mathbf{\mathbf{w}}^*_y))$. 

The final step is to compare the upper bound of target class evidence $\overline{s}_{y^\prime}$ with the lower bound of true class evidence $\underbar{s}_y$. If the condition $\max_{y^\prime\in\mathcal{Y}^\prime}(\overline{s}_{y^\prime})>\underbar{s}_y$ is satisfied, we assume an attack is possible and the algorithm returns \texttt{False}.
On the other hand, if Algorithm~\ref{alg-provable-masking} checks all possible malicious windows $\mathbf{w}\in\mathcal{W}$ for all possible target classes $y^\prime\in\mathcal{Y}^\prime$ and does not return \texttt{False} in any case, this means our defense on this clean image has provable robustness against any possible patch and can always make a correct prediction.
\end{proof}

\noindent\textbf{Provable adversarial training.} We note that our provable analysis can be incorporated into the training process to improve provable robustness. We call this ``provable adversarial training" and will discuss its details in Appendix~\ref{apx-training}.

\begin{algorithm}[t]
\caption{Provable analysis of robust masking}\label{alg-provable-masking}
\begin{algorithmic}[1]
\renewcommand{\algorithmicrequire}{\textbf{Input:}}
\renewcommand{\algorithmicensure}{\textbf{Output:}}
\Require Image $\mathbf{x}$, true class $y$, wrong label set $\mathcal{Y}^\prime = \mathcal{Y}\setminus \{y\}$, feature extractor $\mathcal{F}$ of model $\mathcal{M}$, clipping upper bound $c_h$, the set of sliding windows $\mathcal{W}$, detection threshold $T$.
\Ensure  Whether the image $\mathbf{x}$ has provable robustness
\Procedure{ProvableAnalysisMasking}{}

\For{each $\mathbf{w} \in \mathcal{W}$} 
\LeftCommentb{Upper bound of target class evidence}
\For{each $y^\prime \in \mathcal{Y}^\prime$}
\State $\hat{\mathbf{u}}_{y^\prime} \gets \textsc{Clip}(\mathcal{F}(\mathbf{x},y^\prime),0,c_h)$ 
\State $\overline{s}_{y^\prime}\gets \textsc{Sum}(\hat{\mathbf{u}}_{y^\prime}\odot(\mathbf{1}-\mathbf{w}))/(1-T)$

\EndFor
\LeftCommentb{Lower bound of true class evidence}
\State $\hat{\mathbf{u}}_{y} \gets \textsc{Clip}(\mathcal{F}(\mathbf{x},y),0,c_h)$ 
\State $\mathbf{w}^*_y \gets \textsc{Detect}(\hat{\mathbf{u}}_y\odot(\mathbf{1}-\mathbf{w}),\mathcal{W},T)$
\State $\underbar{s}_y\gets \textsc{Sum}(\hat{\mathbf{u}}_{y}\odot(\mathbf{1}-\mathbf{w})\odot(\mathbf{1}-\mathbf{\mathbf{w}}^*_y))$
\LeftCommentb{Feasibility of an attack}
\If{$\max_{y^\prime\in\mathcal{Y}^\prime}(\overline{s}_{y^\prime})>\underbar{s}_y$}
\State \Return  \texttt{False} 
\EndIf
\EndFor
\State\Return  \texttt{True} 
\EndProcedure

\end{algorithmic} 
\end{algorithm}

\section{Evaluation}\label{sec-evaluation}
In this section, we provide a comprehensive evaluation of \framework. We report the provable robust accuracy of our defense (obtained from Algorithm~\ref{alg-provable-masking} and Theorem~\ref{thm}) on the ImageNet~\cite{deng2009imagenet}, ImageNette~\cite{imagenette}, and CIFAR-10~\cite{krizhevsky2009learning} datasets for various patch sizes. We instantiate our defense with multiple different CNNs with small receptive fields and compare their performance with previous provably robust defenses~\cite{chiang2020certified,zhang2020clipped,levine2020randomized}. We also provide a detailed analysis of our defense performance with different settings.  

\begin{table}[b]
    \centering
        \caption{Default defense settings for Mask-BN and Mask-DS}
        \resizebox{\linewidth}{!}{
    \begin{tabular}{l|l|l}
    \toprule
       Setting  & Feature & Parameters \\
       \midrule
        Mask-BN on ImageNet(te) &  BagNet-17 logits &\multirow{4}{*}{\shortstack{$c_l=0$\\$c_h=\infty$\\$T=0$} }\\
        Mask-BN on CIFAR-10 & BagNet-17 logits & \\
        \cmidrule{1-2}
        Mask-DS on ImageNet(te) & DS-25-ResNet-50 confidence &\\
        Mask-DS on CIFAR-10 &  DS-4-ResNet-18 confidence & \\
        \bottomrule
    \end{tabular}}
    \label{tab-default}
\end{table}

\begin{table*}[t]
    \centering
        \caption{Clean and provable robust accuracy for different defenses}
    \resizebox{\linewidth}{!} {\scriptsize
     \begin{tabular}{c|c|c|c|c|c|c|c|c|c|c|c|c|c|c|c|c}
    \toprule
    Dataset & \multicolumn{6}{c|}{ImageNette} & \multicolumn{6}{c|}{ImageNet} & \multicolumn{4}{c}{CIFAR-10}\\
    \midrule
    Patch size  &  \multicolumn{2}{c|}{1\% pixels}   &  \multicolumn{2}{c|}{2\% pixels} & \multicolumn{2}{c|}{3\% pixels} & \multicolumn{2}{c|}{1\% pixels} &         \multicolumn{2}{c|}{2\% pixels} &          \multicolumn{2}{c|}{3\% pixels} & 
    \multicolumn{2}{c|}{0.4\% pixels}& 
    \multicolumn{2}{c}{2.4\% pixels}  \\
         \midrule
       Accuracy   & clean & robust &clean & robust & clean & robust & clean & robust & clean & robust & clean &  robust&clean& robust &clean& robust \\

    \midrule
    Mask-BN & \textbf{95.2} &\textbf{89.0}& \textbf{95.0}& \textbf{86.7} & 94.8&\textbf{83.0} & \textbf{55.1} & \textbf{32.3} & \textbf{54.6}  & \textbf{26.0}& \textbf{54.1}& \textbf{19.7} &   84.5 &63.8&83.9& {47.3} \\
    Mask-DS&  92.3 & 83.1 &  92.1 & 79.9 & 92.1 & 76.8& 44.1&{19.7}&43.6&{15.7}&43.0&{12.5}  & \textbf{84.7} & \textbf{69.2} &  \textbf{84.6}& \textbf{57.7}  \\
    
    \midrule
    IBP~\cite{chiang2020certified} & \multicolumn{12}{c|}{computationally infeasible}  & 65.8& 51.9& 47.8 & 30.8\\
CBN~\cite{zhang2020clipped}& {94.9} &74.6&{94.9} & 60.9 & \textbf{94.9}& 45.9 &49.5& 13.4&49.5 &7.1  &49.5 &3.1 &{84.2}& 44.2& {84.2}&9.3 \\
    DS~\cite{levine2020randomized}& 92.1& 82.3&92.1&79.1& 92.1& 75.7 &  44.4 & 17.7& 44.4&14.0& 44.4 &11.2 & 83.9&68.9& 83.9&56.2 \\

    \bottomrule
    \end{tabular}}
    \label{tab-huge-provable}
\end{table*}

\subsection{Experiment Setup}\label{sec-exp-setup}
\noindent\textbf{Datasets.} We report our main provable robustness results on the 1000-class ImageNet~\cite{deng2009imagenet}, 10-class ImageNette~\cite{imagenette}, and 10-class CIFAR-10~\cite{krizhevsky2009learning} datasets. ImageNet and ImageNette images have a high resolution and were resized and cropped to 224$\times$224 or 299$\times$299 before being fed into different models while CIFAR-10 images have a lower resolution of 32$\times$32. CIFAR-10 images are rescaled to 192$\times$192 before being fed to BagNet. Further details are in Appendix~\ref{apx-dataset}.

\noindent\textbf{Models.} As discussed in Section~\ref{sec-small-field}, we have two general ways to build a network with small receptive fields. In our evaluation, we instantiate the ensemble approach using a de-randomized smoothed ResNet (DS-ResNet)~\cite{levine2020randomized}, and the small convolution kernel approach using BagNet~\cite{brendel2019approximating}. The DS-ResNet~\cite{levine2020randomized} takes a rectangle pixel patch, or a pixel band, as the input of its base model and uses prediction majority voting for the ensemble prediction. In contrast, our defense uses \emph{robust masking} for aggregation. The BagNet~\cite{brendel2019approximating} architecture replaces a fraction of 3$\times$3 convolution kernels of ResNet-50 with 1$\times$1 kernels to reduce the receptive field size. It was originally proposed in the context of interpretable machine learning while we use this model for provable robustness against adversarial patch attacks. 
\\We analyze performance of ResNet-50, BagNet-33, BagNet-17, BagNet-9, and DS-25-ResNet-50. These 5 models have a similar network structure but have different receptive fields of 483$\times$483, 33$\times$33, 17$\times$17, 9$\times$9, and 25$\times$299, respectively. For CIFAR-10, we additionally include a DS-ResNet-18 with a band size of 4 (DS-4-ResNet-18). Model training details are in Appendix~\ref{apx-training}.

\noindent\textbf{Defenses.} We report the defense performance of our robust masking defense with the BagNet (Mask-BN) and with the DS-ResNet (Mask-DS). We also compare with the existing Clipped BagNet (CBN)~\cite{zhang2020clipped}, De-randomized Smoothing (DS)~\cite{levine2020randomized} and Interval Bound Propagation based certified defense (IBP)~\cite{chiang2020certified}. The default settings of our defense are listed in Table~\ref{tab-default}. Note that for \framework, we use the same set of parameters (i.e., $c_l,c_h,T$) for all datasets and models. For previous defenses, we use the optimal parameter settings obtained from their respective papers.

\noindent\textbf{Attack Patch Size.} For ImageNet and ImageNette, we analyze our defense performance against a single square adversarial patch that consists of up to 1\%, 2\%, or 3\% pixels of the images. For CIFAR-10, we report results for a patch consisting of 0.4\% or 2.4\% of the image pixels. In Appendix~\ref{apx-large-patch}, we analyze the defense performance against larger patches to understand the limits of \framework.

\subsection{Provable Robustness Results}\label{sec-eval-provable}
In this subsection, we present \textit{provable robustness} results for our defense (computed with Algorithm~\ref{alg-provable-masking} and Theorem~\ref{thm}); the results hold for \textit{any attack} within the corresponding patch size constrain. We also compare \framework with previous provably robust defenses~\cite{chiang2020certified,zhang2020clipped,levine2020randomized}.

\noindent \textbf{\framework achieves high provable robustness across different models and datasets.} We report the provable robust accuracy of \framework across different models, patch sizes, and datasets in Table~\ref{tab-huge-provable}. First, both Mask-BN and Mask-DS achieve high provable robustness. For example, against a 1\% pixel patch on the 10-class ImageNette dataset, Mask-BN has a provable robust accuracy of 89.0\% while Mask-DS has that of 83.1\%. This implies that for 89.0\% and 83.1\% of the images from the respective test sets, \emph{no attack using a 1\%-pixel square patch can succeed}. Second, \framework has high provable robustness across different datasets. Even for the extremely challenging 1000-class ImageNet dataset, Mask-BN achieves a non-trivial provable robust accuracy of 32.3\% for the 1\% pixel patch. The provable robust accuracy increases to 54.8\% if we consider the top-5 classification task (more details for the top-k analysis are in Appendix~\ref{apx-topk}).

\noindent \textbf{\framework also maintains high clean accuracy.} As shown in Table~\ref{tab-huge-provable}, \framework retains high clean accuracy. For a 1\% pixel patch, Mask-BN has a 95.2\% clean accuracy on ImageNette and 55.1\% on ImageNet. Mask-DS also has a 92.3\% clean accuracy on ImageNette and 44.1\% on ImageNet. For a 2.4\% pixel patch on CIFAR-10, Mask-BN and Mask-DS have a high clean accuracy of 83.9\% and 84.6\%, respectively. In Table~\ref{tab-huge-vanilla}, we report the clean accuracy of ResNet and BagNet. We can see that the clean accuracy drop of Mask-BN and Mask-DS on ImageNette compared with undefended ResNet is within 7.5\%. The accuracy drop of Mask-BN from the undefended BagNet is within 1\%.\footnote{BagNet alone does not have any provable robustness but acts as a building block for the provable defense of \framework.}

We note that we use the optimal mask window sizes for different estimated upper bounds of patch sizes, and therefore the clean accuracy for different patches varies slightly in Table~\ref{tab-huge-provable}. We will show a similarly high performance of our defense when using an over-conservatively large mask window size in Section~\ref{sec-eval-detail-masking}.

\begin{table}[t]
    \centering
        \caption{Clean accuracy of ResNet and BagNet for different datasets}
   \resizebox{0.8\linewidth}{!}  {\scriptsize
     \begin{tabular}{c|c|c|c}
    \toprule
    Dataset & {ImageNette} & {ImageNet} & {CIFAR-10}\\
         \midrule
         ResNet &  99.6\% & 76.1\% &  97.0\%\\
        BagNet &  95.9\% &  56.5\% & 85.4\%  \\
    \bottomrule
    \end{tabular}}

    \label{tab-huge-vanilla}
\end{table}
\noindent\textbf{\framework achieves higher provable robust accuracy than all previous defenses.} We compare our defense performance with existing defenses across three datasets. 

\textit{Comparison with IBP~\cite{chiang2020certified}.} IBP is too computationally expensive and does not scale to high-resolution images like ImageNette and ImageNet. We thus only compare its performance with PatchGuard on CIFAR-10. As shown in Table~\ref{tab-huge-provable}, both Mask-BN and Mask-DS significantly outperform IBP in terms of provable robust accuracy and clean accuracy.

\textit{Comparison with CBN~\cite{zhang2020clipped}.} Table~\ref{tab-huge-provable} shows that both Mask-BN and Mask-DS have higher provable robust accuracy than CBN across three datasets. The clean accuracy of Mask-BN is higher or comparable with that of CBN, but its provable robust accuracy is much higher. For example, against a 3\% pixel patch on ImageNette, Mask-BN (94.8\%) has a similar clean accuracy as CBN (94.9\%), but its provable robust accuracy is 37.1\% higher! 

\textit{Comparison with DS~\cite{levine2020randomized}.} Both Mask-BN and Mask-DS have better defense performance than DS on the high-resolution ImageNette and ImageNet datasets. For example, against a 1\% pixel patch on ImageNet, Mask-BN has a 10.7\% higher clean accuracy and a 14.6\% higher provable robust accuracy compared with DS. On CIFAR-10, Mask-DS outperforms DS in terms of clean accuracy and provable robust accuracy thanks to the robust masking defense. 

\noindent\textbf{Takeaways.} Our evaluation shows the effectiveness of our proposed defenses, achieving state-of-the-art provable robustness on all three datasets. We find that BagNet-based defenses (Mask-BN and CBN) perform well on ImageNette and ImageNet but are fragile on CIFAR-10 due to the low image resolution. Meanwhile, De-randomized Smoothing based defenses (Mask-DS and DS) perform better on CIFAR-10. This shows that while the \textit{robust masking} defense always improves robustness, the choice of which model to use (Mask-BN or Mask-DS) depends on the dataset.

\subsection{Detailed Analysis of \framework}\label{sec-eval-detail}

In this subsection, we analyze the behavior of vanilla (undefended) models, \framework with different parameters, and defense efficiency on the ImageNette dataset. We will only report results for Mask-BN when the observations from Mask-BN and Mask-DS are very similar. A similar analysis for CIFAR-10 is available in Appendix~\ref{apx-cifar}.

\begin{table}[t]
    \centering
        \caption{Effect of logits clipping values on vanilla models}
   \resizebox{\linewidth}{!}{ \scriptsize
   \begin{tabular}{c|c|c|c|c|c}
    \toprule
    $(c_l,c_h)$& $(-\infty,\infty)$&$(0,\infty)$&$(0,50)$ & $(0,15)$ & $(0,5)$ \\
    \midrule
    ResNet-50 & 99.6\% & 99.5\% & 99.5\% & 99.5\% & 99.0\%\\
    BagNet-33 &97.2\% & 97.1\%& 97.0\% & 95.8\% & 94.1\%\\
    BagNet-17 & 95.9\% & 95.5\%& 94.7\% & 92.3\% & 87.9\%\\
    BagNet-9 &92.5\% & 92.5\% & 91.4\% & 85.4\% & 73.8\%\\

    \bottomrule
    \end{tabular}}
    \label{tab-clipping}
\end{table}

\begin{table}[t]
    \centering
        \caption{Invariance of BagNet-17 predictions to feature masking}
   \resizebox{\linewidth}{!}{ \scriptsize
   \begin{tabular}{c|c|c|c|c|c}
    \toprule
    Window size& 0$\times$0&2$\times$2&4$\times$4 & 6$\times$6 & 8$\times$8 \\
    \midrule
    Masked accuracy &95.9\% & 95.9\% & 95.9\% & 95.8\% & 95.7\%\\
    \% images   & 4.1\% & 5.1\% &6.1\% & 7.3\%& 8.5\%\\
    \% windows per image& 0\% & 0.05\% & 0.2\% & 0.4\% &0.7\%\\
    \bottomrule
    \end{tabular}}
    \label{tab-feature-mask}
\end{table}

\subsubsection{Analysis of Vanilla models}\label{sec-eval-detail-vanilla}
Recall that \framework's robust prediction relies on clipping feature values as well as robust masking. Here, we show that vanilla models only have a small performance loss due to clipping and feature masking, which explains the high clean accuracy retained by \framework.

\noindent\textbf{Clipping has a small impact on vanilla models.} In this analysis, we vary the clipping value for the local logits for ResNet and BagNet to determine how the clean accuracy changes, and the results are shown in Table~\ref{tab-clipping}. We find that clipping the negative values only slightly affects the clean accuracy ($c_l=0,c_h=\infty$ is our default setting). When we decrease the positive clipping value $c_h$, the clean accuracy of the model also decreases. We notice that models with smaller receptive fields are more sensitive to clipping. This is because models with small receptive fields only have a small number of correct local predictions. The corresponding correctly predicted local logits have to use large logits values to dominate the global prediction, which leads to the sensitivity to clipping. As shown in Figure~\ref{fig-hist-net}, the logits of the adversarial images tend to have large values. If we set $c_h$ to the largest clean logits value, we will not affect the clean accuracy and can improve the \textit{empirical} robustness against the adversarial patch.

\noindent\textbf{Vanilla models are generally prediction-invariant to feature masking.} In our \emph{robust masking} defense, we detect and mask corrupted features. If the model can make correct predictions from the aggregation of the remaining features, we can recover the correct prediction. We use BagNet-17, which has $26\cdot26$ local features, to analyze the prediction invariance of vanilla models to partial feature masking. We mask out all class evidence within a set of sliding windows of different sizes and record the prediction from the remaining features. We report the average accuracy over all possible masked feature tensors (masked accuracy), the percentage of images for which at least one masked prediction is incorrect (\% images), and the averaged percentage of masks that will cause prediction change for each image (\% windows per image).\footnote{We note that ``\% images" presented in Table~\ref{tab-feature-mask} is an upper bound for our robust masking in Algorithm~\ref{alg-masking} because robust masking masks the window with the highest class evidence for each class while this analysis only removed wrong class evidence within the same window as the true class.} As shown in Table~\ref{tab-feature-mask}, the overall average masked accuracy is high, and the percentage of images and windows for which the prediction changes is low. Such a small fraction of images with prediction changes enables us to achieve high provable robustness and maintain clean accuracy.

\begin{table}[t]
    \centering
        \caption{Effect of receptive field sizes on provable robust accuracy}
    \resizebox{\linewidth}{!}{\scriptsize
    \begin{tabular}{c|c|c|c|c|c|c}
    \toprule
          Patch size &  \multicolumn{2}{c|}{1\% pixels} &\multicolumn{2}{c|}{2\% pixels} & \multicolumn{2}{c}{3\% pixels}  \\
         \midrule
     Accuracy    & clean &robust& clean &robust& clean &robust\\
         \midrule
     Mask-BN-33 & 96.5\% & 88.9\% & 96.3\% & 86.0\% & 96.3\% &82.1\%\\
     Mask-BN-17 & {95.2}\% &{89.0}\%& {95.0}\%& {86.7}\% & 94.8\%& {83.0}\%\\
     Mask-BN-9 & 92.1\%  & 85.5\% &91.8\% & 82.8\% &91.5\% & 79.8\%\\
     \bottomrule
    \end{tabular}}
    \label{tab-masking-provable}
\end{table}

\subsubsection{\framework with Different Parameters}\label{sec-eval-detail-masking}
\noindent\textbf{The receptive field size balances the trade-off between clean accuracy and provable robust accuracy of defended models.} We report clean accuracy and provable robust accuracy of our defense with BagNet-33, BagNet-17, and BagNet-9, which have a receptive field of 33$\times$33, 17$\times$17, and 9$\times$9, respectively, against different patch sizes in Table~\ref{tab-masking-provable}. As shown in the table, a model with a larger receptive field has better clean accuracy. However, a larger receptive field results in a larger fraction of corrupted features and thus a larger gap between clean accuracy and provable robust accuracy. We can see that though Mask-BN-33 has a higher clean accuracy than Mask-BN-17, its gap between clean accuracy and provable robust accuracy is larger, which results in a similar or slightly poorer provable robust accuracy compared with Mask-BN-17. \emph{The trade-off between the clean accuracy and the robustness can be tuned with different receptive field sizes and should be carefully balanced when deploying the defense.}

\begin{table}[t]
    \centering
        \caption{Effect of detection thresholds on Mask-BN-17}   
   \resizebox{\linewidth}{!} { \scriptsize
   \begin{tabular}{c|c|c|c}
    \toprule
          & Clean accuracy & Provable accuracy & Detection FP  \\
         \midrule
     T-0.0 & 95.0\% & 86.7\%& 100\%    \\
     T-0.2 & 94.2\%  &79.9\%  & 22.9\%   \\
     T-0.4 & 95.3\%  &68.0\% & 0.7\%  \\
     T-0.6 &  95.5\%   &38.7\% & 0.05\% \\
     T-0.8 & 95.5\% & 6.2\%&  0\%   \\
     T-1.0 & 95.5\% & 0\% &0\%   \\
     \bottomrule
    \end{tabular}}
 \label{tab-masking-thres}
\end{table}

\noindent\textbf{A large detection threshold improves clean accuracy but decreases provable robust accuracy of defended models.} We study the model performance of BagNet-17 against a 2\% pixel patch as we change the detection threshold $T$ from $0.0$ to $1.0$. A threshold of zero means our detection will always return a suspicious window even if the input is a clean image while a threshold of one means no detection at all. We report the clean accuracy, provable robust accuracy, and false positive (FP) rates for detection of suspicious windows on clean images in Table~\ref{tab-masking-thres}. As we increase the detection threshold $T$, we reduce the FP rate for clean images, at the cost of making it easier for an adversarial patch to succeed via \textit{Case IV} (no suspicious window detected). \emph{However, we note that false positives in the detection phase for clean images have a minimal impact on the clean accuracy because our models are generally invariant to feature masking, as already shown in Table~\ref{tab-feature-mask}.} Thus, we find $T=0$ to be the best choice for this dataset (even with an FP of 100\%); it results in the highest provable robust accuracy of 86.7\% while only incurring a 0.5\% clean accuracy drop compared to $T=1$.

\begin{table}[t]
    \centering
        \caption{Effect of feature types on Mask-BN-17}
   \resizebox{\linewidth}{!}  {\scriptsize
    \begin{tabular}{c|c|c|c|c|c|c}
    \toprule
      Patch size     &  \multicolumn{2}{c|}{1\% pixels} &\multicolumn{2}{c|}{2\% pixels} & \multicolumn{2}{c}{3\% pixels}  \\
         \midrule
      Accuracy   & clean &robust& clean &robust& clean &robust\\
         \midrule
     Logits&  {95.2}\% &{89.0}\%& {95.0}\%& {86.7}\% & 94.8\%& {83.0}\%\\
     Confidence & 87.9\% & 80.5\% & 87.9\% & 77.9\% & 88.0\% & 74.4\%\\
     Prediction & 85.7\% & 77.3\% & 85.8\% & 74.1\% & 85.9\% & 70.3\%\\
     \bottomrule
    \end{tabular}}
    \label{tab-feature-bn}
\end{table}
\begin{table}[t]
    \centering
        \caption{Effect of feature types on Mask-DS}
  \resizebox{\linewidth}{!} {\scriptsize
    \begin{tabular}{c|c|c|c|c|c|c}
    \toprule
      Patch size     &  \multicolumn{2}{c|}{1\% pixels} &\multicolumn{2}{c|}{2\% pixels} & \multicolumn{2}{c}{3\% pixels}  \\
         \midrule
     Accuracy    & clean &robust& clean &robust& clean &robust\\
         \midrule
       
     Logits&92.4\%& 76.9\% & 92.1\% & 68.9\% & 91.9\% & 61.6\% \\
     Confidence & 92.3\% & 83.1\% &  92.1\% & 79.9\% & 92.1\% & 76.8\% \\
     Prediction & 91.9\% & 82.5\% & 91.8\% & 79.4\% & 91.7\%& 76.4\%\\
     \bottomrule
    \end{tabular}}
    \label{tab-feature-ds}
\end{table}

\begin{table}[t]
    \centering
        \caption{Effect of over-conservatively large masks on Mask-BN-17}
    \resizebox{\linewidth}{!} {\scriptsize
    \begin{tabular}{c|c|c|c|c}
    \toprule
     \backslashbox{mask}{patch}     & clean &  1\% pixels &2\% pixels & 3\% pixels  \\
         \midrule
    1\% pixels& 95.2\% & {89.0\% }& -- & -- \\ 
     2\% pixels & 95.0\%&88.2\% & {86.7\%}&-- \\
     3\% pixels & 94.8\% & 87.1\% & 85.3\% & {83.0\%}\\
     4.5\% pixels & 94.6\%& 86.0\% &  84.1\% & 81.8\%\\
     \midrule
         CBN~\cite{zhang2020clipped}&{94.9\%} &74.6\%& 60.9\% &  45.9\% \\
    DS~\cite{levine2020randomized}& 92.1\%& 82.3\%&79.1\%&  75.7\%  \\
     \bottomrule
    \end{tabular}}
    \label{tab-mismatch}
\end{table}

\begin{table}[t]
    \centering
        \caption{Per-image inference time of different models}
\resizebox{\linewidth}{!}{\scriptsize
    \begin{tabular}{c|c|c|c|c|c}
    \toprule
 Model & ResNet-50 & BagNet-17& DS-25-ResNet & Mask-BN & Mask-DS\\
  \midrule
Time  & 11.8ms & 12.1ms &387.9ms & 16.6ms&  404.4ms\\
     \bottomrule
    \end{tabular}}
    \label{tab-time}
\end{table}
\noindent\textbf{Different feature types greatly influence the performance of defended models.} In this analysis, we study the performance of the robust masking defense when using different types of features, namely logits, confidence values, and predictions. The results for Mask-BN-17 with different features are reported in Table~\ref{tab-feature-bn}. As shown in the table, using logits as the feature type has much better performance than confidence and prediction in terms of clean accuracy and provable accuracy. The main reason for this observation is that BagNet is trained with logits aggregation. Our additional analysis shows that BagNet does not have high model performance when trained with confidence or prediction aggregation; therefore, we use logits as our default feature type for Mask-BN. Interestingly, Mask-DS exhibits a different behavior. As shown in Table~\ref{tab-feature-ds}, Mask-DS works better when we use prediction or confidence as feature types due to its different training objectives. In conclusion, the performance of different feature types largely depends on the training objective of the network with small receptive fields, and should be appropriately optimized to determine the best defense setting. 

\noindent\textbf{Over-conservatively large masks only have a small impact on defended models.} \framework's robust masking is deployed in a manner that is agnostic to the patch size by selecting a large mask window size that matches the upper bound of the patch size. In this analysis, we study the model performance when an over-conservatively large mask is used. Note that the provable robustness obtained with a larger mask for a larger patch can be directly applied to a smaller patch (e.g., an image that is robust against a 3\% pixel patch is also robust against a 1\% pixel patch). However, we can certify the robustness for more images when the actual patch size is smaller than the mask size (Appendix~\ref{apx-mismatch}). 

We report the provable robust accuracy and clean accuracy of Mask-BN-17 with different patch sizes and attack-agnostic mask sizes in Table~\ref{tab-mismatch}. First, robust masking with a larger mask can have a tighter provable robustness bound for a smaller patch. For example, when using a 3\% pixel mask, the provable analysis in Algorithm~\ref{alg-provable-masking} (using Lemma~\ref{lemma}) can only certify the robustness of 83.0\% of test images for any patch size smaller than 3\%. In contrast, the tighter provable analysis from Appendix~\ref{apx-mismatch} leads to a provable robust accuracy of 87.1\% (4.1\% improvement) for a 1\% pixel patch. Second, over-conservatively using a larger mask size only leads to a slight drop in clean accuracy and provable robust accuracy. As we increase the mask size, the clean accuracy for 1\% pixel patch only drops from 95.2\% to 94.6\% and the provable robust accuracy drops from 89.0\% to 86.0\%. We note that even when the mismatch is large (a 4.5\% pixel mask for a 1\% pixel patch), our defense still outperforms DS~\cite{levine2020randomized}.

\subsubsection{Defense Efficiency}\label{sec-eval-detail-efficiency}
\noindent \textbf{Robust masking only introduces a small defense overhead.}
In Table~\ref{tab-time}, we report the per-image inference time of different models on the ImageNette validation set. As shown in the table, the inference time of Mask-BN (16.6ms) is close to that of BagNet-17 (12.1ms). We have a similar observation for Mask-DS (404.4ms) and DS-25-ResNet (387.9ms).

\noindent \textbf{BagNet-like models (e.g., Mask-BN) are more efficient than DS-like models (e.g., DS and Mask-DS).} As discussed in Section~\ref{sec-small-field}, using an ensemble model (e.g., DS-ResNet) is computationally expensive compared with using small convolution kernels in conventional CNNs (e.g., BagNet). From Table~\ref{tab-time}, we can see the inference time of BagNet-17 (12.1ms) much smaller than that of DS-25-ResNet (387.9ms). This difference leads to a huge efficiency gap between Mask-BN (16.6ms) and Mask-DS (404.4ms) as well as DS (387.9ms). Therefore, we suggest using small convolution kernels to build models with small receptive fields when the two approaches have similar defense performance. 

\section{Discussion}\label{sec-discussion}
In this section, we will show that \framework is a generalization of other provable defenses and discuss its limitations and future directions.

\subsection{Generalization of Related Defenses}\label{sec-generalization-cbn-ds}
In this subsection, we will show that our defense framework is a generalization of other provably robust defenses such as Clipped BagNet~\cite{zhang2020clipped}, De-randomized Smoothing~\cite{levine2020randomized}.

\noindent\textbf{Clipped BagNet (CBN).} CBN~\cite{zhang2020clipped} proposes clipping the local logits tensor with function $\textsc{Clip}(\mathbf{u})=\tanh(0.05\cdot \mathbf{u} - 1)$ to improve the robustness of BagNet~\cite{brendel2019approximating}. 
Since the range of $\tanh(\cdot)$ is bounded by $(-1,1)$, the adversary can achieve at most $2k$ difference in clipped logits values between the true class and any other class, where $k$ is the number of corrupted local logits due to the adversarial patch. In its provable analysis, CBN calculates the difference between the sum of unaffected logits values for the predicted class and the second predicted class as $\delta$; if $\delta>2 k$, CBN certifies the robustness of the input clean image. To reduce our Mask-BN defense to CBN, we can set our feature type to logits, the detection threshold to $T=1$ (i.e., no detection), and adjust the clipping values $c_l$ and $c_h$ or the clipping function $\textsc{Clip}(\cdot)$. Our evaluation shows that our defense significantly outperforms CBN across three different datasets. There are two major reasons for this performance difference: 1) CBN \emph{retains} the malicious feature values while \framework \emph{detects and masks them}; 2) CBN uses conventional training while \framework uses provable adversarial training (Appendix~\ref{apx-training}).

\noindent\textbf{De-randomized Smoothing (DS).} DS~\cite{levine2020randomized} trains a `smoothed' classifier on image pixel patches and computes the predicted class as the class with the majority vote among local predictions made from all pixel patches. The provable robustness analysis of DS only considers the largest and second-largest counts of local predictions. If the gap between the two largest counts is larger than $2 k$, where $k$ is the upper bound of the number of corrupted predictions, DS certifies the robustness of the image.
When we set the feature type to prediction and detection threshold to $T=1$ (i.e., no detection), we can reduce Mask-DS to DS. Note that averaging all one-hot encoded local predictions gives the same global prediction as majority voting. The major cause of the relatively poor performance of DS is that its certification process discards the spatial information of each prediction while our robust masking defense utilizes the spatial information that all corrupted features are within a small window in the feature space. We provide a more detailed comparison in Appendix~\ref{apx-generalization}.

We note that two defenses (BagCert~\cite{metzen2021efficient} and Randomized Cropping~\cite{lin2021certified}) appeared after the initial release of our paper preprint~\cite{xiang2020patchguard}; both of them can be regarded as instances of our \framework framework, i.e., using CNNs with small receptive fields (modified BagNet~\cite{metzen2021efficient}; image cropping~\cite{lin2021certified}) and secure aggregation (majority voting\cite{metzen2021efficient,lin2021certified}). These two followup works further demonstrate the generality of \framework.

\subsection{Limitations and Future Work}\label{sec-future}
While \framework achieves state-of-the-art provable robustness and has higher or comparable clean accuracy compared with previous defenses, there is still a drop in clean accuracy compared with undefended models. We note that \framework is compatible with any small-receptive-field CNN and secure aggregation mechanism, and we expect the trade-off between provable robustness and clean accuracy to be mitigated further given any progress in these two directions.

\noindent\textbf{CNNs with small receptive fields.} The use of small receptive fields provides substantial provable robustness but incurs a non-negligible clean accuracy drop for the two architectures (i.e., BagNet~\cite{brendel2019approximating} and DS-ResNet~\cite{levine2020randomized}) used in this paper. In future work, we aim to explore better architectures and training methods for CNNs with small receptive fields in order to provide robustness against patch attacks while maintaining state-of-the-art clean accuracy. Any progress on this front will directly boost our defense performance since \framework is compatible with any CNN with small receptive fields.


\noindent\textbf{Secure feature aggregation.} We present \emph{robust masking} to compute robust predictions from partially corrupted features. Robust masking works in a manner that is agnostic to the patch size by using a large mask, but a completely parameter-free defense may be more desirable. To this end, we observe that \framework turns the problem of designing an adversarial patch defense into a robust aggregation problem, i.e., \emph{how can we make a robust prediction from a partially corrupted feature tensor?} Thus, techniques from robust statistics such as median, truncated mean, as well as differential privacy~\cite{dwork2014algorithmic} can also be incorporated in our framework, some of which admit a parameter-free defense. We also plan to explore the design of custom secure aggregation mechanisms in future work that can further improve provable robustness.

\section{Related Work}\label{sec-related-work}

\subsection{Localized Adversarial Perturbations}\label{related-work-attack}
Most adversarial example research focuses on global $L_p$-norm bounded perturbations while localized adversaries have received much less attention. The adversarial patch attack was introduced by Brown et
al. \cite{brown2017adversarial} and focused on physical and universal patches to induce targeted misclassification. Attacks in the real-world can be realized by attaching a patch to the victim object.  A follow-up paper on Localized and Visible Adversarial Noise (LaVAN) attack~\cite{karmon2018lavan} aimed at inducing targeted misclassification in the digital domain.  

Localized patch attacks against object detection \cite{liu2018dpatch,thys2019fooling}, semantic segmentation models \cite{sehwag2018not} as well as training-time poisoning attacks using localized triggers \cite{gu2017badnets,liu2017trojaning} have been proposed. Our threat model in this paper focuses on attacks against image classification models at test time; how to generalize our defense to the above settings can be an interesting future direction to study. 


\subsection{Adversarial Patch Defenses}

Empirical defenses like Digital Watermark (DW)~\cite{hayes2018visible} and Local Gradient Smoothing (LGS)~\cite{naseer2019local} were first proposed to detect and neutralize adversarial patch. However, these heuristic defenses are vulnerable to adaptive attackers with knowledge of the defense.

Observing the ineffectiveness of DW and LGS, Chiang et al.~\cite{chiang2020certified} proposed the first provable defense against adversarial patches via Interval Bound Propagation (IBP)~\cite{gowal2018effectiveness,mirman2018differentiable}. Despite its important theoretical contribution, the IBP defense has poor clean and provable robust accuracy, as shown in Table~\ref{tab-huge-provable}. Zhang et al.~\cite{zhang2020clipped} proposed clipped BagNet (CBN) for provable robustness and Levine et al.~\cite{levine2020randomized} proposed building a `smoothed' classifier (DS) that outputs the class with the largest count from local predictions on all small pixel patches. We have shown that CBN and DS are instances of our general defense framework (Section~\ref{sec-generalization-cbn-ds}), and \framework has better performance due to the use of robust masking (Section~\ref{sec-eval-provable}). The Minority Report (MR)~\cite{mccoyd2020minority} defense was proposed in concurrent work, where the defender puts a mask at all possible locations and extracts patterns from model predictions. This defense can only provably detect an attack while \framework also guarantees the recovery of the correct prediction. Moreover, MR performs  masking in the image space which is computationally expensive and cannot scale to high-resolution images. However, if we can tolerate attack detection, MR has an advantage on low-resolution images (90.6\% clean accuracy and 62.1\% provable accuracy for 2.4\%-pixel patch on CIFAR-10; compared to our 84.6\% clean accuracy and 57.7\% provable accuracy). How to extend \framework for attack detection is an interesting direction of future work.

Another concurrent line of research has been on adversarial patch training~\cite{wu2019defending,rao2020adversarial}. However, these works focus on empirical robustness and do not provide any provable guarantees.

\subsection{Receptive Fields of CNNs}
A number of papers have studied the influence of the receptive field \cite{brendel2019approximating,araujo2019computing,le2017receptive,luo2016understanding} on model performance in order to better understand the model behavior. BagNet \cite{brendel2019approximating} adopted the structure of ResNet-50 \cite{he2016deep} but reduced the receptive field size by replacing 3$\times$3 kernels with 1$\times$1 kernels. BagNet-17 can achieve similar top-5 validation accuracy as AlexNet~\cite{krizhevsky2012imagenet} on ImageNet~\cite{deng2009imagenet} dataset when each feature only looks at a 17$\times$17 pixel region. The small receptive field was used for better interpretability of model decisions in the original BagNet paper. In this work, we use the reduced receptive field size to create models robust to adversarial patch attacks.

\subsection{Other Adversarial Example Attacks and Defenses}
The development of adversarial example-based attacks and defenses has been an extremely active research area over the past few years. Conventional adversarial attacks~\cite{szegedy2013intriguing,goodfellow2014explaining,papernot2016limitations,carlini2017towards} craft adversarial examples that have a small $L_p$ distance to clean examples but induce model misclassification. Many empirical defenses~\cite{papernot2016distillation,xu2017feature,meng2017magnet,metzen2017detecting} have been proposed to address the adversarial example vulnerability, but most of them can be easily bypassed by strong adaptive attackers~\cite{carlini2017adversarial,athalye2018obfuscated,tramer2020adaptive}. The fragility of the empirical defenses has inspired provable or certified defenses~\cite{raghunathan2018certified,wong2017provable,lecuyer2019certified,cohen2019certified,salman2019provably,gowal2018effectiveness,mirman2018differentiable} as well as work on learning-theoretic bounds in the presence of adversaries \cite{bhagoji2019lower,cullina2018pac,pmlr-v97-dohmatob19a,pmlr-v97-yin19b,schmidt2018adversarially}. In contrast, the focus of this paper is on localized adversarial patch attacks, and we refer interested readers to survey papers~\cite{papernot2018sok,yuan2019adversarial} for a more detailed background on adversarial examples. 
\section{Conclusion}
In this paper, we propose a general provable defense framework called \framework that mitigates localized adversarial patch attacks. We identify large receptive fields and insecure aggregation mechanisms in conventional CNNs as the key sources of vulnerability to adversarial patches. To address these two problems, our defense proposes the use of models with small receptive fields to limit the number of features corrupted by the adversary which are then augmented with a \emph{robust masking} defense to detect and mask the corrupted features to ensure secure feature aggregation. Our defense achieves state-of-the-art provable robust accuracy on ImageNet, ImageNette, and CIFAR-10 datasets. We hope that our general defense framework inspires further research to fully mitigate adversarial patch attacks.

\section*{Acknowledgements}
We are grateful to David Wagner for shepherding the paper and anonymous reviewers at USENIX Security for their valuable feedback.
This work was supported in part by the National Science Foundation under grants CNS-1553437 and CNS-1704105, the ARL’s Army Artificial Intelligence Innovation Institute (A2I2), the Office of Naval Research Young Investigator Award, the Army Research Office Young Investigator Prize, Faculty research award from Facebook, Schmidt DataX award, and Princeton E-ffiliates Award.

\bibliographystyle{plain}
\bibliography{patch_defense}

\begin{thebibliography}{10}

\bibitem{araujo2019computing}
Andre Araujo, Wade Norris, and Jack Sim.
\newblock Computing receptive fields of convolutional neural networks.
\newblock {\em Distill}, 2019.
\newblock https://distill.pub/2019/computing-receptive-fields.

\bibitem{athalye2018obfuscated}
Anish Athalye, Nicholas Carlini, and David~A. Wagner.
\newblock Obfuscated gradients give a false sense of security: Circumventing
  defenses to adversarial examples.
\newblock In {\em Proceedings of the 35th International Conference on Machine
  Learning ({ICML})}, pages 274--283, 2018.

\bibitem{bhagoji2019lower}
Arjun~Nitin Bhagoji, Daniel Cullina, and Prateek Mittal.
\newblock Lower bounds on adversarial robustness from optimal transport.
\newblock In {\em Conference on Neural Information Processing Systems
  (NeurIPS)}, pages 7496--7508, 2019.

\bibitem{bagnet}
Wieland Brendel.
\newblock Pretrained bag-of-local-features neural networks.
\newblock \url{https://github.com/wielandbrendel/bag-of-local-features-models},
  2020.

\bibitem{brendel2019approximating}
Wieland Brendel and Matthias Bethge.
\newblock Approximating {CNNs} with bag-of-local-features models works
  surprisingly well on {ImageNet}.
\newblock In {\em 7th International Conference on Learning Representations
  ({ICLR})}, 2019.

\bibitem{brown2017adversarial}
Tom~B. Brown, Dandelion Man{\'{e}}, Aurko Roy, Mart{\'{\i}}n Abadi, and Justin
  Gilmer.
\newblock Adversarial patch.
\newblock In {\em Conference on Neural Information Processing Systems Workshops
  ({NeurIPS} Workshops)}, 2017.

\bibitem{carlini2017adversarial}
Nicholas Carlini and David~A. Wagner.
\newblock Adversarial examples are not easily detected: Bypassing ten detection
  methods.
\newblock In {\em Proceedings of the 10th {ACM} Workshop on Artificial
  Intelligence and Security (AISec@CCS)}, pages 3--14, 2017.

\bibitem{carlini2017towards}
Nicholas Carlini and David~A. Wagner.
\newblock Towards evaluating the robustness of neural networks.
\newblock In {\em 2017 {IEEE} Symposium on Security and Privacy ({S\&P})},
  pages 39--57, 2017.

\bibitem{chiang2020certified}
Ping-Yeh Chiang, Renkun Ni, Ahmed Abdelkader, Chen Zhu, Christoph Studor, and
  Tom Goldstein.
\newblock Certified defenses for adversarial patches.
\newblock In {\em 8th International Conference on Learning Representations
  ({ICLR})}, 2020.

\bibitem{cohen2019certified}
Jeremy~M. Cohen, Elan Rosenfeld, and J.~Zico Kolter.
\newblock Certified adversarial robustness via randomized smoothing.
\newblock In {\em Proceedings of the 36th International Conference on Machine
  Learning ({ICML})}, pages 1310--1320, 2019.

\bibitem{cullina2018pac}
Daniel Cullina, Arjun~Nitin Bhagoji, and Prateek Mittal.
\newblock {PAC}-learning in the presence of adversaries.
\newblock In {\em Conference on Neural Information Processing Systems
  (NeurIPS)}, pages 230--241, 2018.

\bibitem{deng2009imagenet}
Jia Deng, Wei Dong, Richard Socher, Li{-}Jia Li, Kai Li, and Fei{-}Fei Li.
\newblock {ImageNet}: {A} large-scale hierarchical image database.
\newblock In {\em 2009 {IEEE} Computer Society Conference on Computer Vision
  and Pattern Recognition {(CVPR})}, pages 248--255, 2009.

\bibitem{pmlr-v97-dohmatob19a}
Elvis Dohmatob.
\newblock Generalized no free lunch theorem for adversarial robustness.
\newblock In {\em Proceedings of the 36th International Conference on Machine
  Learning (ICML)}, pages 1646--1654, 2019.

\bibitem{dwork2014algorithmic}
Cynthia Dwork and Aaron Roth.
\newblock The algorithmic foundations of differential privacy.
\newblock {\em Foundations and Trends in Theoretical Computer Science},
  9(3-4):211--407, 2014.

\bibitem{dc}
Jeremy Elson, John~(JD) Douceur, Jon Howell, and Jared Saul.
\newblock Asirra: A captcha that exploits interest-aligned manual image
  categorization.
\newblock In {\em Proceedings of 14th ACM Conference on Computer and
  Communications Security (CCS)}. Association for Computing Machinery, Inc.,
  October 2007.

\bibitem{evtimov2017robust}
Kevin Eykholt, Ivan Evtimov, Earlence Fernandes, Bo~Li, Amir Rahmati, Chaowei
  Xiao, Atul Prakash, Tadayoshi Kohno, and Dawn Song.
\newblock Robust physical-world attacks on deep learning visual classification.
\newblock In {\em 2018 {IEEE} Conference on Computer Vision and Pattern
  Recognition ({CVPR})}, pages 1625--1634, 2018.

\bibitem{imagenette}
fast.ai.
\newblock {ImageNette}: A smaller subset of 10 easily classified classes from
  imagenet.
\newblock \url{https://github.com/fastai/imagenette}, 2020.

\bibitem{wordnet}
Ingo Feinerer and Kurt Hornik.
\newblock {\em wordnet: WordNet Interface}, 2017.
\newblock R package version 0.1-14.

\bibitem{goodfellow2014explaining}
Ian~J. Goodfellow, Jonathon Shlens, and Christian Szegedy.
\newblock Explaining and harnessing adversarial examples.
\newblock In {\em 3rd International Conference on Learning Representations
  ({ICLR})}, 2015.

\bibitem{gowal2018effectiveness}
Sven Gowal, Krishnamurthy Dvijotham, Robert Stanforth, Rudy Bunel, Chongli Qin,
  Jonathan Uesato, Relja Arandjelovic, Timothy~Arthur Mann, and Pushmeet Kohli.
\newblock Scalable verified training for provably robust image classification.
\newblock In {\em 2019 {IEEE/CVF} International Conference on Computer Vision
  ({ICCV})}, pages 4841--4850, 2019.

\bibitem{gu2017badnets}
Tianyu Gu, Brendan Dolan-Gavitt, and Siddharth Garg.
\newblock {BadNets}: Identifying vulnerabilities in the machine learning model
  supply chain.
\newblock In {\em Machine Learning and Computer Security Workshop (NeurIPS
  MLSec)}, 2017.

\bibitem{hayes2018visible}
Jamie Hayes.
\newblock On visible adversarial perturbations {\&} digital watermarking.
\newblock In {\em 2018 {IEEE} Conference on Computer Vision and Pattern
  Recognition Workshops ({CVPR} Workshops)}, pages 1597--1604, 2018.

\bibitem{he2016deep}
Kaiming He, Xiangyu Zhang, Shaoqing Ren, and Jian Sun.
\newblock Deep residual learning for image recognition.
\newblock In {\em 2016 {IEEE} Conference on Computer Vision and Pattern
  Recognition ({CVPR})}, pages 770--778, 2016.

\bibitem{karmon2018lavan}
Danny Karmon, Daniel Zoran, and Yoav Goldberg.
\newblock {LaVAN}: Localized and visible adversarial noise.
\newblock In {\em Proceedings of the 35th International Conference on Machine
  Learning (ICML)}, pages 2512--2520, 2018.

\bibitem{krizhevsky2009learning}
Alex Krizhevsky.
\newblock Learning multiple layers of features from tiny images.
\newblock \url{https://www.cs.toronto.edu/~kriz/learning-features-2009-TR.pdf},
  2009.

\bibitem{krizhevsky2012imagenet}
Alex Krizhevsky, Ilya Sutskever, and Geoffrey~E Hinton.
\newblock {ImageNet} classification with deep convolutional neural networks.
\newblock In {\em Conference on Neural Information Processing Systems
  (NeurIPS)}, pages 1106--1114, 2012.

\bibitem{le2017receptive}
Hung Le and Ali Borji.
\newblock What are the receptive, effective receptive, and projective fields of
  neurons in convolutional neural networks?
\newblock {\em arXiv preprint arXiv:1705.07049}, 2017.

\bibitem{lecuyer2019certified}
Mathias L{\'{e}}cuyer, Vaggelis Atlidakis, Roxana Geambasu, Daniel Hsu, and
  Suman Jana.
\newblock Certified robustness to adversarial examples with differential
  privacy.
\newblock In {\em 2019 {IEEE} Symposium on Security and Privacy ({S\&P})},
  pages 656--672, 2019.

\bibitem{ds}
Alexander Levine and Soheil Feizi.
\newblock Code for the paper ``(de)randomized smoothing for certifiable defense
  against patch attacks".
\newblock \url{https://github.com/alevine0/patchSmoothing}, 2020.

\bibitem{levine2020randomized}
Alexander Levine and Soheil Feizi.
\newblock {(De)}randomized smoothing for certifiable defense against patch
  attacks.
\newblock In {\em Conference on Neural Information Processing Systems,
  ({NeurIPS})}, 2020.

\bibitem{lin2021certified}
Wan-Yi Lin, Fatemeh Sheikholeslami, jinghao shi, Leslie Rice, and J~Zico
  Kolter.
\newblock Certified robustness against physically-realizable patch attack via
  randomized cropping, 2021.

\bibitem{liu2018dpatch}
Xin Liu, Huanrui Yang, Ziwei Liu, Linghao Song, Yiran Chen, and Hai Li.
\newblock {DPATCH:} an adversarial patch attack on object detectors.
\newblock In {\em Workshop on Artificial Intelligence Safety 2019 co-located
  with the 33rd {AAAI} Conference on Artificial Intelligence 2019 (AAAI)},
  volume 2301, 2019.

\bibitem{liu2017trojaning}
Yingqi Liu, Shiqing Ma, Yousra Aafer, Wen{-}Chuan Lee, Juan Zhai, Weihang Wang,
  and Xiangyu Zhang.
\newblock Trojaning attack on neural networks.
\newblock In {\em 25th Annual Network and Distributed System Security Symposium
  ({NDSS})}, 2018.

\bibitem{luo2016understanding}
Wenjie Luo, Yujia Li, Raquel Urtasun, and Richard~S. Zemel.
\newblock Understanding the effective receptive field in deep convolutional
  neural networks.
\newblock In {\em Conference on Neural Information Processing Systems
  (NeurIPS)}, pages 4898--4906, 2016.

\bibitem{madry2017towards}
Aleksander Madry, Aleksandar Makelov, Ludwig Schmidt, Dimitris Tsipras, and
  Adrian Vladu.
\newblock Towards deep learning models resistant to adversarial attacks.
\newblock In {\em 6th International Conference on Learning Representations
  ({ICLR})}, 2018.

\bibitem{mccoyd2020minority}
Michael McCoyd, Won Park, Steven Chen, Neil Shah, Ryan Roggenkemper, Minjune
  Hwang, Jason~Xinyu Liu, and David~A. Wagner.
\newblock Minority reports defense: Defending against adversarial patches.
\newblock In {\em Applied Cryptography and Network Security Workshops ({ACNS}
  Workshops)}, volume 12418, pages 564--582. Springer, 2020.

\bibitem{meng2017magnet}
Dongyu Meng and Hao Chen.
\newblock Magnet: {A} two-pronged defense against adversarial examples.
\newblock In {\em Proceedings of the 2017 {ACM} {SIGSAC} Conference on Computer
  and Communications Security ({CCS})}, pages 135--147, 2017.

\bibitem{metzen2017detecting}
Jan~Hendrik Metzen, Tim Genewein, Volker Fischer, and Bastian Bischoff.
\newblock On detecting adversarial perturbations.
\newblock In {\em 5th International Conference on Learning Representations
  ({ICLR})}, 2017.

\bibitem{metzen2021efficient}
Jan~Hendrik Metzen and Maksym Yatsura.
\newblock Efficient certified defenses against patch attacks on image
  classifiers.
\newblock In {\em 9th International Conference on Learning Representations
  ({ICLR})}, 2021.

\bibitem{mirman2018differentiable}
Matthew Mirman, Timon Gehr, and Martin~T. Vechev.
\newblock Differentiable abstract interpretation for provably robust neural
  networks.
\newblock In {\em Proceedings of the 35th International Conference on Machine
  Learning ({ICML})}, pages 3575--3583, 2018.

\bibitem{naseer2019local}
Muzammal Naseer, Salman Khan, and Fatih Porikli.
\newblock Local gradients smoothing: Defense against localized adversarial
  attacks.
\newblock In {\em {IEEE} Winter Conference on Applications of Computer Vision
  ({WACV})}, pages 1300--1307, 2019.

\bibitem{papernot2018sok}
Nicolas Papernot, Patrick McDaniel, Arunesh Sinha, and Michael~P Wellman.
\newblock Sok: Security and privacy in machine learning.
\newblock In {\em 2018 IEEE European Symposium on Security and Privacy
  (EuroS\&P)}, pages 399--414, 2018.

\bibitem{papernot2016limitations}
Nicolas Papernot, Patrick~D. McDaniel, Somesh Jha, Matt Fredrikson, Z.~Berkay
  Celik, and Ananthram Swami.
\newblock The limitations of deep learning in adversarial settings.
\newblock In {\em {IEEE} European Symposium on Security and Privacy
  (EuroS{\&}P)}, pages 372--387, 2016.

\bibitem{papernot2016distillation}
Nicolas Papernot, Patrick~D. McDaniel, Xi~Wu, Somesh Jha, and Ananthram Swami.
\newblock Distillation as a defense to adversarial perturbations against deep
  neural networks.
\newblock In {\em {IEEE} Symposium on Security and Privacy ({S\&P})}, pages
  582--597, 2016.

\bibitem{pytorch}
Adam Paszke, Sam Gross, Francisco Massa, Adam Lerer, James Bradbury, Gregory
  Chanan, Trevor Killeen, Zeming Lin, Natalia Gimelshein, Luca Antiga, Alban
  Desmaison, Andreas K{\"{o}}pf, Edward Yang, Zachary DeVito, Martin Raison,
  Alykhan Tejani, Sasank Chilamkurthy, Benoit Steiner, Lu~Fang, Junjie Bai, and
  Soumith Chintala.
\newblock Pytorch: An imperative style, high-performance deep learning library.
\newblock In {\em Conference on Neural Information Processing Systems
  (NeurIPS)}, pages 8024--8035, 2019.

\bibitem{raghunathan2018certified}
Aditi Raghunathan, Jacob Steinhardt, and Percy Liang.
\newblock Certified defenses against adversarial examples.
\newblock In {\em 6th International Conference on Learning Representations
  ({ICLR})}, 2018.

\bibitem{rao2020adversarial}
Sukrut Rao, David Stutz, and Bernt Schiele.
\newblock Adversarial training against location-optimized adversarial patches.
\newblock In {\em European Conference on Computer Vision Workshops ({ECCV}
  Workshops)}, 2020.

\bibitem{salman2019provably}
Hadi Salman, Jerry Li, Ilya~P. Razenshteyn, Pengchuan Zhang, Huan Zhang,
  S{\'{e}}bastien Bubeck, and Greg Yang.
\newblock Provably robust deep learning via adversarially trained smoothed
  classifiers.
\newblock In {\em Conference on Neural Information Processing Systems
  (NeurIPS)}, pages 11289--11300, 2019.

\bibitem{schmidt2018adversarially}
Ludwig Schmidt, Shibani Santurkar, Dimitris Tsipras, Kunal Talwar, and
  Aleksander Madry.
\newblock Adversarially robust generalization requires more data.
\newblock In {\em Advances in Neural Information Processing Systems}, pages
  5014--5026, 2018.

\bibitem{sehwag2018not}
Vikash Sehwag, Chawin Sitawarin, Arjun~Nitin Bhagoji, Arsalan Mosenia, Mung
  Chiang, and Prateek Mittal.
\newblock Not all pixels are born equal: An analysis of evasion attacks under
  locality constraints.
\newblock In {\em Proceedings of the 2018 ACM SIGSAC Conference on Computer and
  Communications Security (CCS)}, pages 2285--2287, 2018.

\bibitem{simonyan2014very}
Karen Simonyan and Andrew Zisserman.
\newblock Very deep convolutional networks for large-scale image recognition.
\newblock In {\em 3rd International Conference on Learning Representations
  ({ICLR})}, 2015.

\bibitem{szegedy2017inception}
Christian Szegedy, Sergey Ioffe, Vincent Vanhoucke, and Alexander~A. Alemi.
\newblock Inception-v4, inception-resnet and the impact of residual connections
  on learning.
\newblock In {\em Proceedings of the Thirty-First {AAAI} Conference on
  Artificial Intelligence ({AAAI})}, pages 4278--4284, 2017.

\bibitem{szegedy2015going}
Christian Szegedy, Wei Liu, Yangqing Jia, Pierre Sermanet, Scott Reed, Dragomir
  Anguelov, Dumitru Erhan, Vincent Vanhoucke, and Andrew Rabinovich.
\newblock Going deeper with convolutions.
\newblock In {\em Proceedings of the IEEE Conference on Computer Vision and
  Pattern Recognition (CVPR)}, pages 1--9, 2015.

\bibitem{szegedy2013intriguing}
Christian Szegedy, Wojciech Zaremba, Ilya Sutskever, Joan Bruna, Dumitru Erhan,
  Ian~J. Goodfellow, and Rob Fergus.
\newblock Intriguing properties of neural networks.
\newblock In {\em 2nd International Conference on Learning Representations
  ({ICLR})}, 2014.

\bibitem{tfflowers}
The~TensorFlow Team.
\newblock Flowers.
\newblock
  \url{http://download.tensorflow.org/example_images/flower_photos.tgz}, Jan
  2019.

\bibitem{thys2019fooling}
Simen Thys, Wiebe~Van Ranst, and Toon Goedem{\'{e}}.
\newblock Fooling automated surveillance cameras: Adversarial patches to attack
  person detection.
\newblock In {\em {IEEE} Conference on Computer Vision and Pattern Recognition
  Workshops ({CVPR} Workshops)}, pages 49--55, 2019.

\bibitem{pretrained}
Torchvision.
\newblock torchvision.models.
\newblock \url{https://pytorch.org/docs/stable/torchvision/models.html}, 2020.

\bibitem{tramer2020adaptive}
Florian Tramer, Nicholas Carlini, Wieland Brendel, and Aleksander Madry.
\newblock On adaptive attacks to adversarial example defenses.
\newblock In {\em 2020 USENIX Security and AI Networking Summit (ScAINet)},
  2020.

\bibitem{wong2017provable}
Eric Wong and J.~Zico Kolter.
\newblock Provable defenses against adversarial examples via the convex outer
  adversarial polytope.
\newblock In {\em Proceedings of the 35th International Conference on Machine
  Learning ({ICML})}, pages 5283--5292, 2018.

\bibitem{wu2019defending}
Tong Wu, Liang Tong, and Yevgeniy Vorobeychik.
\newblock Defending against physically realizable attacks on image
  classification.
\newblock In {\em 8th International Conference on Learning Representations
  ({ICLR})}, 2020.

\bibitem{xiang2020patchguard}
Chong Xiang, Arjun~Nitin Bhagoji, Vikash Sehwag, and Prateek Mittal.
\newblock Patchguard: Provable defense against adversarial patches using masks
  on small receptive fields.
\newblock {\em arXiv preprint arXiv:2005.10884v1}, 2020.

\bibitem{xu2017feature}
Weilin Xu, David Evans, and Yanjun Qi.
\newblock Feature squeezing: Detecting adversarial examples in deep neural
  networks.
\newblock In {\em 25th Annual Network and Distributed System Security Symposium
  ({NDSS})}, 2018.

\bibitem{pmlr-v97-yin19b}
Dong Yin, Ramchandran Kannan, and Peter Bartlett.
\newblock Rademacher complexity for adversarially robust generalization.
\newblock In {\em Proceedings of the 36th International Conference on Machine
  Learning (ICML)}, pages 7085--7094, 2019.

\bibitem{yuan2019adversarial}
Xiaoyong Yuan, Pan He, Qile Zhu, and Xiaolin Li.
\newblock Adversarial examples: Attacks and defenses for deep learning.
\newblock {\em IEEE Transactions on Neural Networks and Learning Systems},
  30(9):2805--2824, 2019.

\bibitem{zhang2020clipped}
Zhanyuan Zhang, Benson Yuan, Michael McCoyd, and David Wagner.
\newblock Clipped bagnet: Defending against sticker attacks with clipped
  bag-of-features.
\newblock In {\em 3rd Deep Learning and Security Workshop (DLS)}, 2020.

\end{thebibliography}

\newpage

\appendix

\section{Empirical Adversarial Patch Attacks} \label{apx-attack}
We provide additional details of empirical adversarial patch attacks in this section. Note that we only use the empirical attacks to evaluate the \textit{undefended} models; in contrast, our defended models are evaluated via provable robustness metrics.

\noindent \textbf{Attack algorithm.} The untargeted adversarial patch attack can be formulated as the following optimization problem, as done in related literature~\cite{brown2017adversarial,karmon2018lavan}.
\begin{equation}\label{eqn-untargeted}
    \mathbf{x}^\prime = \arg\max_{\mathbf{x}^\prime \in \mathcal{A}(\mathbf{x})} \mathcal{L}(\mathcal{M}(\mathbf{x}^\prime),y)
\end{equation}
Here we abuse the notation to let $\mathcal{M}(\mathbf{x}^\prime)$ to be the final predicted confidence vector of the model, and $y$ to be the one-hot encoded vector of the class. $\mathcal{L(\cdot)}$ refers to the cross-entropy loss. For the targeted attack with target class $y^\prime \neq y$, the optimization objective is slightly different as:
\begin{equation}
    \mathbf{x}^\prime = \arg\min_{\mathbf{x}^\prime \in \mathcal{A}(\mathbf{x})} \mathcal{L}(\mathcal{M}(\mathbf{x}^\prime),y^\prime)
\end{equation}
Since we have $\mathbf{x}^\prime = (\mathbf{1}-\mathbf{p})\odot \mathbf{x} + \mathbf{p} \odot \mathbf{x}^{\prime\prime}$, and $\mathbf{p},\mathbf{x}$ are fixed, the actual optimization is over $\mathbf{x}^{\prime\prime}$, which is distinguished from conventional $L_p$ adversary optimization~\cite{carlini2017towards,madry2017towards}. This optimization problem can be approximately solved with gradient-based optimization algorithms such as Projected Gradient Descent~\cite{madry2017towards}.

\noindent\textbf{Attack setup.} We use untargeted LaVAN attack (formulated in Equation~\ref{eqn-untargeted}) for our case study in Section~\ref{sec-motivation}. Due to the computational constraint, we only select five random locations to perform the attack against each image. For each image and each random location, we use a 500-step Projected Gradient Descent~\cite{madry2017towards} with a learning rate of 0.05 to generate the adversarial patch. We record and report the best attack result for each image. As shown in Section~\ref{sec-motivation}, the untargeted LaVAN attack with five random locations are already effective enough against the undefended vanilla models, though the attack performance can be further improved if the adversary tries all possible locations. We note that our provable analysis considers all possible patch locations.  

\section{Details of Experiment Setup}\label{apx-setup}

In this section, we provide additional details for datasets and model training. 

\subsection{Datasets}\label{apx-dataset}
\noindent\textbf{ImageNet.} ImageNet~\cite{deng2009imagenet} is a popular benchmark dataset for high-resolution image classification. It has 1281167 training images and 50000 validation images from 1000 classes organized according to the WordNet~\cite{wordnet} hierarchy. It is conventional to resize and crop ImageNet images to 224$\times$224 or 299$\times$299. For ResNet, we take 224$\times$224 images as inputs and use the pre-trained model from \cite{pretrained}. For BagNet, we take 224$\times$224 inputs and fine-tune pre-trained models from \cite{bagnet}. For DS-ResNet-50, we use 299$\times$299 images and use pre-trained models from \cite{ds}. The actual input of the base classifier of DS-ResNet-50 is a pixel patch in the shape of 25x299; DS-ResNet-50 counts the predictions from all possible 25$\times$299 pixel patches and predicts the class with the largest count. When the confidence of local prediction is lower than $0.2$, the base model will abstain from prediction. 

\noindent\textbf{ImageNette.} ImageNette~\cite{imagenette} is a subset of ImageNet with 9469 training images and 3925 validation images from classes of \texttt{tench}, \texttt{English springer}, \texttt{cassette player}, \texttt{chain saw}, \texttt{church}, \texttt{French horn}, \texttt{garbage truck}, \texttt{gas pump}, \texttt{golf ball}, \texttt{parachute}. We use this smaller dataset for a more comprehensive defense evaluation. We use the same model architectures as the models for ImageNet, but modify the last fully-connected layer (i.e., the classification layer) to accommodate for 10-class classification. We use the pre-trained models~\cite{pretrained,bagnet,ds}, retrain the entire model for 20 epochs with a batch size of 64, and retain the model with the highest validation accuracy. We use Stochastic Gradient Descent (SGD) with a 0.001 initial learning rate and a 0.9 momentum for model training. 

\noindent\textbf{CIFAR-10.} CIFAR-10~\cite{krizhevsky2009learning} is a benchmark dataset for low-resolution image classification. CIFAR-10 has 50000 training images and 10000 test images from classes \texttt{plane}, \texttt{car}, \texttt{bird}, \texttt{cat}, \texttt{deer}, \texttt{dog}, \texttt{frog}, \texttt{horse}, \texttt{ship}, \texttt{truck}. Each image is in the shape of 32$\times$32. For ResNet, we train the model from scratch for 20 epochs. We use SGD with a 0.01 initial learning rate, a 0.9 momentum, and a 5e-4 weight decay. We reduce the learning rate with a factor of 0.1 every 10 epochs. For BagNet, we re-scale the 32$\times$32 images to 192$\times$192 with bicubic interpolation because we find BagNet does not have good performance on low-resolution images. We use the models~\cite{bagnet} pre-trained on ImageNet and retrain the entire model with the same hyperparameters as we retrain a BagNet. For DS-ResNet-18, we take a 4$\times$32 pixel band as the input to its base classifier and use the pre-trained model from \cite{ds}. Note that DS-ResNet-18 is based on the structure of ResNet-18 while other models are based on ResNet-50. The reported clean accuracy of ResNet-50 in Table~\ref{tab-huge-vanilla} is obtained from ResNet-50 retrained on 192$\times$192 images, facilitating comparison with the clean accuracy of BagNet-17 trained on re-scaled images.

\begin{table}[t]
    \centering
        \caption{Effect of provable adversarial training on Mask-BN-17} 
  \resizebox{\linewidth}{!} { \scriptsize
    \begin{tabular}{c|c|c|c|c|c|c}
    \toprule
          Dataset &  \multicolumn{2}{c|}{ImageNet} &\multicolumn{2}{c|}{ImageNette}  &\multicolumn{2}{c}{CIFAR-10}  \\
         \midrule
     Accuracy    & clean &robust& clean &robust& clean &robust \\
         \midrule

     Conventional training &54.4\% & 13.3\% & 93.9\% & 83.8\% & 82.6\% & 31.7\%\\
     Provable adv. training & 54.6\% & 26.0\% & 95.0\% & 86.7\% & 83.9\% & 47.3\%\\

     \bottomrule
    \end{tabular}}
 \label{tab-adv-training}
\end{table}

\subsection{Model Training}\label{apx-training}

\noindent\textbf{BagNet.} BagNet~\cite{brendel2019approximating} has a similar architecture as ResNet~\cite{he2016deep} but uses many 1$\times$1 convolution kernels to reduce the receptive field size. Therefore, BagNet can be trained conventionally with the cross-entropy loss. However, to improve the provable robust accuracy, \framework trains a BagNet with a mask over the region with the largest true class evidence. This training mimics the procedure of our provable analysis in Algorithm~\ref{alg-provable-masking}, and we call it \textit{provable adversarial training}. For each BagNet used in \framework, we first train the model in the conventional manner as introduced in Appendix~\ref{apx-dataset} and then do provable adversarial training with a 6$\times$6 feature masks for additional 20 epochs. In Table~\ref{tab-adv-training}, we report the results for Mask-BN-17 with and without provable adversarial training against a 2\% pixel patch on ImageNet/ImageNette and a 2.4\% pixel patch on CIFAR-10. We can see from the table that provable adversarial training significantly improves provable robustness.

\noindent\textbf{DS-ResNet.} DS-ResNet~\cite{levine2020randomized} takes all possible pixel bands as the inputs to its base model, and this feedforward inference is too expensive to perform during the model training. Therefore, following Levine et al.~\cite{levine2020randomized}, we train the base model with a subset of possible pixel bands for efficiency. Since only a fraction of local features are calculated during the training, the provable adversarial training used for BagNet is inapplicable to DS-ResNet.

\subsection{Hardware and Software}
All of our experiments are done on a workstation with 32 Intel Xeon E5-2620 v4 CPU cores, 252GB RAM, and 8 NVIDIA Tesla P100 GPU. Timing results are obtained using one GPU and one CPU core. The deep learning models are implemented in PyTorch~\cite{pytorch}.

\section{Additional Results on Local Logits Analysis of Adversarial Images}\label{apx-local-attack}

In section~\ref{sec-motivation}, we show that the adversarial patches tend to create abnormally large local feature values to dominate the global prediction. In Figure~\ref{fig-hist-nette} and Figure~\ref{fig-hist-cifar}, we provide additional results for ImageNette and CIFAR-10. We can see that these results are similar to 
our observation for ImageNet. These results motivate our \emph{robust masking} defense for a secure feature aggregation.

\begin{figure}[t]
    \centering
    \includegraphics[width=0.7\linewidth]{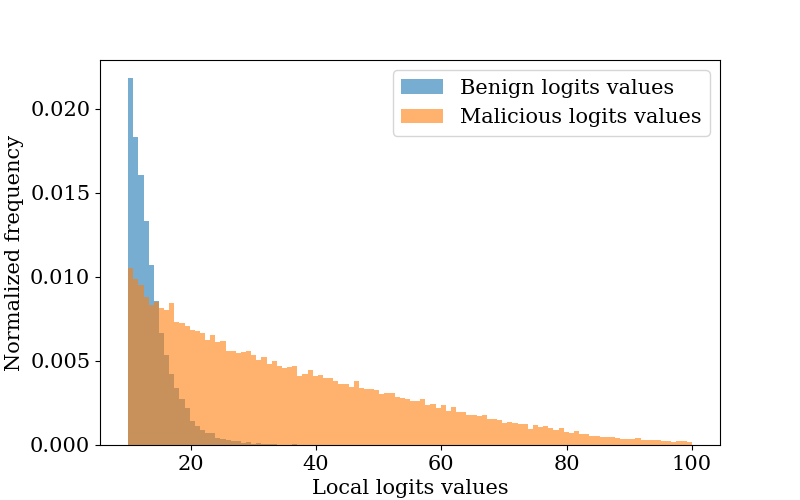}
    \caption{Histogram of large benign and malicious local logits values for ImageNette adversarial images.}
    \label{fig-hist-nette}
\end{figure}

\begin{figure}[t]
    \centering
    \includegraphics[width=0.7\linewidth]{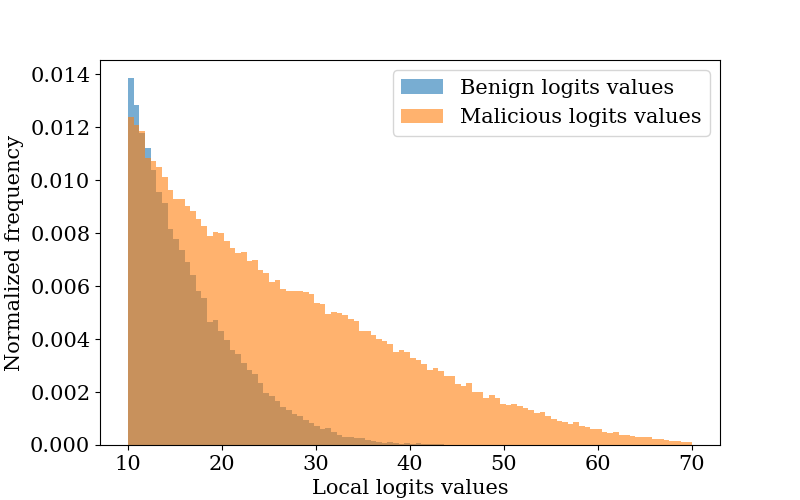}
    \caption{Histogram of large benign and malicious local logits values for CIFAR-10 adversarial images.}
    \label{fig-hist-cifar}
\end{figure}

\begin{figure}[b]
    \centering
    \includegraphics[width=\linewidth]{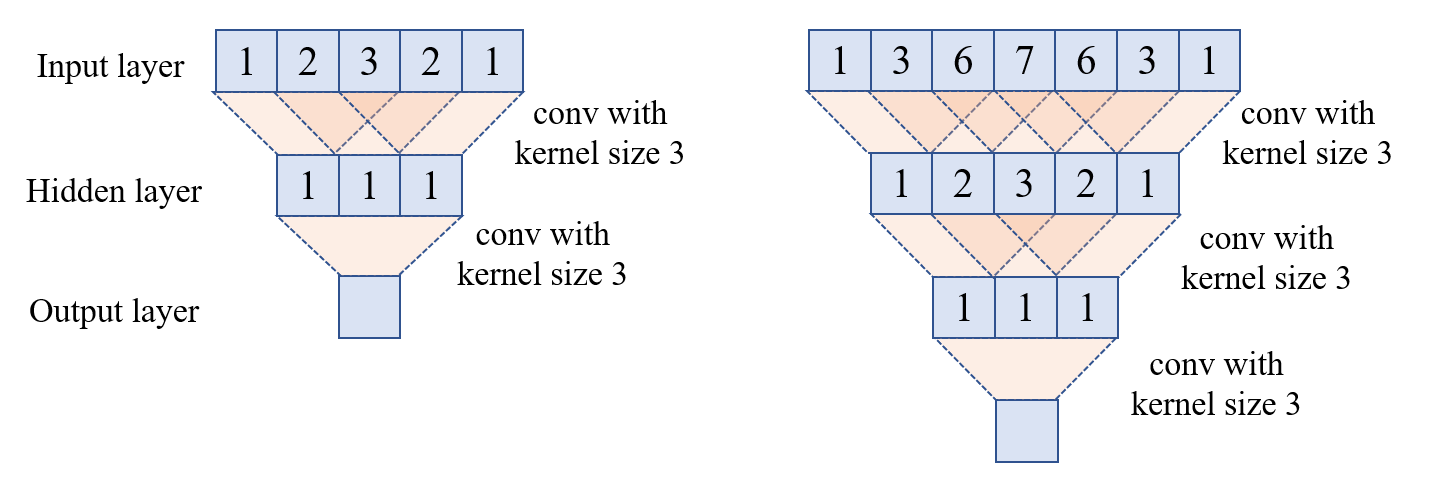}
    \caption{Toy example of 1-D convolution computation}
    \label{fig-center}
\end{figure}

\section{Details of Receptive Fields}\label{apx-receptive}
\noindent\textbf{Local features focus on the center of the receptive field.} In Section~\ref{sec-motivation}, we mentioned that a particular local feature focuses exponentially more on the center of its receptive field. We provide the intuition for this argument in Figure~\ref{fig-center}. The left part of the figure illustrates a 1-D example of convolution computation in which the input has five cells and will go through two convolution layers with a kernel size of 3 to compute the final output. Each cell in the hidden layer (i.e., the output of the first convolution layer) looks at 3 input cells, and the output cell looks at three hidden cells. We count the number of times each cell is looked at when computing the output cell and plot it in the figure. As we can see, the center cell of the input layer receives the most attention (being looked at  3 times). Moreover, as the number of layers increases (a similar example for 3 convolution layers is plotted in the right part of Figure~\ref{fig-center}), the difference in attention between the center cell and the rightmost/leftmost cell will increase exponentially. Therefore, a particular feature focuses exponentially more on the center of its receptive field, and an adversary controlling the center cell will have a larger capacity to manipulate the final output features.

\begin{figure}[t]
    \centering
    \includegraphics[width=\linewidth]{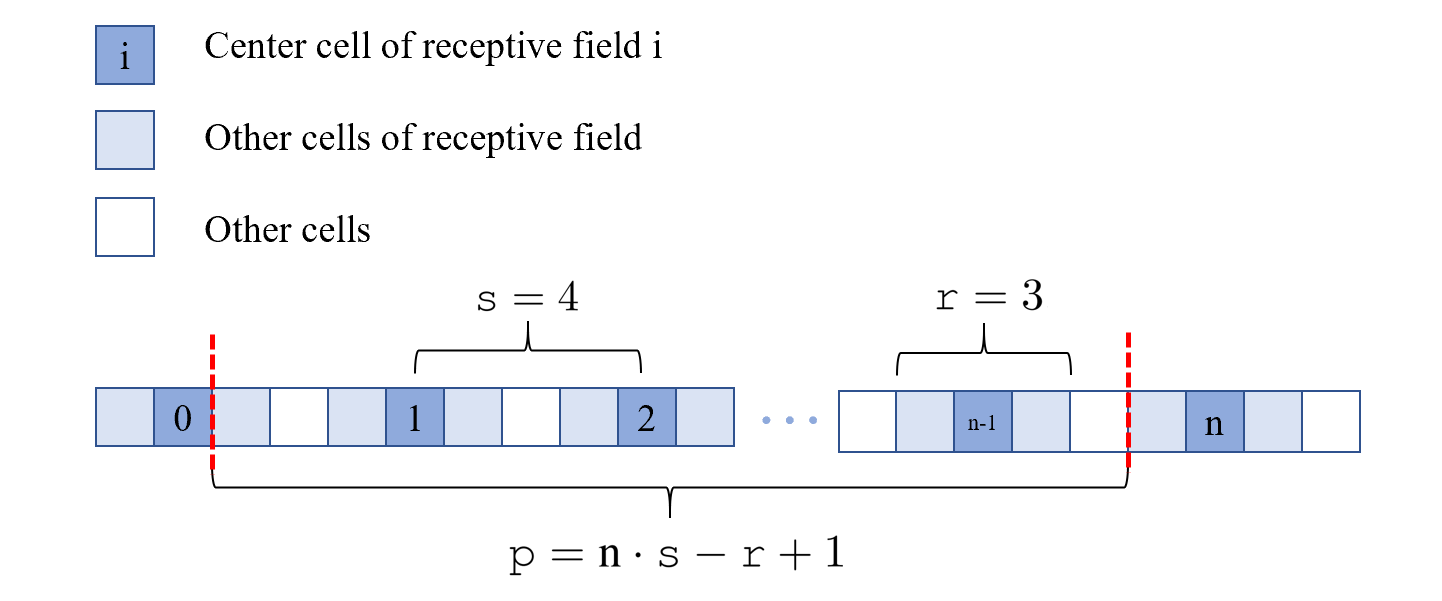}
    \caption{Example of computing window size.}
    \label{fig-window}
\end{figure}

\noindent\textbf{Computing the Window Size.} One crucial step of our robust masking defense is to determine the window size, and we show in Section~\ref{sec-masking} and Equation~\ref{eq-receptive} that the window size \texttt{w} can be computed as $\texttt{w} = \lceil(\texttt{p}+\texttt{r}-1)/\texttt{s}\rceil$, where \texttt{p} is the upper bound on the patch size, \texttt{r} is the size of receptive field, and \texttt{s} is the stride of receptive field. In Figure~\ref{fig-window}, we provide the intuition for Equation~\ref{eq-receptive}. In this example, we assume the stride $\texttt{s}=4$, and the size of the receptive field $\texttt{r}=3$. We distinguish the centers of receptive fields, the other cells in the receptive fields, and the other cells with different colors. Note that we choose a large stride \texttt{s} such that adjacent receptive fields do not overlap for a better visual demonstration; the derived equation is applicable to smaller \texttt{s} or larger \texttt{r}. In Figure~\ref{fig-window}, we want to determine the largest patch size \texttt{p} such that the patch only appears in n but not n+1 receptive fields. We plot the boundary of the largest patch with red dash line in the figure. The left part of the patch covers the rightmost cells of receptive field 0, and the right part does not appear in receptive field n. Based on Figure~\ref{fig-window}, we can compute $\texttt{p}=\text{n}\cdot \texttt{s}-\texttt{r}+1$. Next, we can substitute n with \texttt{w}, use $\lceil\cdot\rceil$ for generalization to any patch size, and finally get $\texttt{w} = \lceil(\texttt{p}+\texttt{r}-1)/\texttt{s}\rceil$. We note that the network architectures~\cite{bagnet,ds} used in this paper have $\texttt{s}=8$ for BagNet and $\texttt{s}=1$ for DS-ResNet.

\section{Tighter Provable Analysis for Over-conservative Mask Size}\label{apx-mismatch}

Recall that the mask window size is a tunable security parameter. In this section, we discuss a tighter provable analysis of robust masking when the defender overestimates the worst-case patch size and use a larger mask window size.

\noindent\textbf{Overestimation.} The provable robustness obtained for a large patch with a large mask is directly applicable to a small patch since the small patch is covered by the large patch. However, we can obtain a tighter bound for the scenario of smaller patch sizes.

As introduced in Section~\ref{sec-provable}, Algorithm~\ref{alg-provable-masking} compares the upper bound of class evidence of every wrong class with the lower bound of true class evidence to certify the robustness. We only need to re-evaluate these bounds in the context of patch size overestimation. Let $\mathcal{W}$ be the set of all possible malicious windows and $\mathcal{V}$ be the set of all possible detected windows whose sizes are larger than malicious windows. We can have the following generalized Lemma.

\begin{lemma}\label{lemma2}
Given a malicious window $\mathbf{w} \in \mathcal{W}$, a class $\bar{y} \in \mathcal{Y}$, and the set of all possible detected windows $\mathcal{V}$, the clipped and masked class evidence of class $\bar{y}$ (i.e., $s_{\bar{y}}$) can be no larger than $\textsc{Sum}(\hat{\mathbf{u}}_{\bar{y}}\odot (\mathbf{1}-\mathbf{v}^*_\mathbf{w}))/(1-T)$, where $\mathbf{v}^*_\mathbf{w}=\arg\max_{\mathbf{v}\in\mathcal{V}_\mathbf{w}}\textsc{Sum}(\hat{\mathbf{u}}_{\bar{y}}\odot \mathbf{v})$ and $\mathcal{V}_\mathbf{w}=\{\mathbf{v}\in\mathcal{V}|\textsc{Sum}(\mathbf{w}\odot\mathbf{v}) = \textsc{Sum}(\mathbf{w})\}$.
\end{lemma}
In this lemma, $\mathcal{V}_\mathbf{w}$ is the set of possible mask windows that cover the whole malicious window $\mathbf{w}$, and $\mathbf{v}^*_\mathbf{w}$ is the mask window in $\mathcal{V}_\mathbf{w}$ with the largest class evidence. This bound reduces to the bound of Lemma~\ref{lemma} when the sizes of malicious window and mask window (i.e., the output of subprocedure \textsc{Detect}) are the same: $\mathcal{V}_\mathbf{w}=\{\mathbf{w}\}$ and thus $\mathbf{v}^*_\mathbf{w}=\mathbf{w}$. The proof of Lemma~\ref{lemma2} is provided below, which is in the same spirit as that of Lemma~\ref{lemma}.

\renewcommand*{\proofname}{Proof of Lemma~\ref{lemma2}}

\begin{proof}
We use $\mathcal{W}$ to denote the set of malicious windows and $\mathcal{V}$ for the set of mask windows. Following the same notations used in the proof of Lemma~\ref{lemma}, we let $e$ be the amount of class evidence within the malicious window $\mathbf{w}$, $t=\textsc{Sum}(\hat{\mathbf{u}}_{\bar{y}}\odot (\mathbf{1}-\mathbf{w}))$ be the class evidence outside $\mathbf{w}$, $e^\prime$ be the class evidence within the detected window $\mathbf{v}^*_{\bar{y}}$. We will have the clipped and masked class evidence $s_{\bar{y}}=t+e-e^\prime$. Next, let $\mathcal{V}_\mathbf{w}$ be the set of possible detected windows that cover the entire malicious window $\mathbf{w}$, and $\mathbf{v}^*_\mathbf{w}$ be the detected window in $\mathcal{V}_\mathbf{w}$ with the largest class evidence. We use $k=\textsc{Sum}(\hat{\mathbf{u}}_{\bar{y}}\odot\mathbf{v}^*_\mathbf{w}\odot(\mathbf{1}-\mathbf{w}))$ to denote the class evidence inside $\mathbf{v}^*_\mathbf{w}$ but outside $\mathbf{w}$. With these notations, we are able to determine the upper bound by considering four possible cases of the detected window $\mathbf{v}^*_{\bar{y}}$.
\begin{enumerate}
    \item \textit{Case I: the malicious window is perfectly detected.} In this case, we have $\mathbf{v}^*_\mathbf{w}=\mathbf{v}^*_{\bar{y}}$ and thus $e^\prime = e+k$. The class evidence $s_{\bar{y}}=t+e-e^\prime = t-k$. 
    \item \textit{Case II: a benign window is incorrectly detected.} In this case, we have $e^\prime=\textsc{Sum}(\hat{\mathbf{u}}_{\bar{y}}\odot\mathbf{v}^*_{\bar{y}})$. The adversary has the constraint that $e+k\leq e^\prime$; otherwise, the window $\mathbf{v}^*_\mathbf{w}$ instead of $\mathbf{v}^*_{\bar{y}}$ will be detected. Therefore, we have $s_{\bar{y}}=t+e-e^\prime \leq t-k$.
    \item \textit{Case III: the malicious window is partially detected.} Let $\mathbf{r}_1 = \mathbf{v}^*_{\bar{y}}\odot (\mathbf{1}-\mathbf{w})$ be the detected benign region, $\mathbf{r}_2=\mathbf{v}^*_{\bar{y}}\odot \mathbf{w}$ be the detected malicious region, and $\mathbf{r}_3=(\mathbf{1}-\mathbf{v}^*_{\bar{y}})\odot \mathbf{w}$ be the undetected malicious region. Let $q_1,q_2,q_3$ be the class evidence within region $\mathbf{r}_1,\mathbf{r}_2,\mathbf{r}_3$, respectively. We have $e=q_2+q_3$ and $e^\prime=q_1+q_2$. Similar to \textit{Case II}, the adversary has the constraint that $e+k \leq e^\prime$, or $q_3 \leq q_1$; otherwise, $\mathbf{v}^*_\mathbf{w}$ instead of $\mathbf{w}^*_{\bar{y}}$ will be detected. Therefore, we have $s_{\bar{y}}=t+e-e^\prime \leq t-k$.
    \item \textit{Case IV: no suspicious window detected.} This case happens when the largest sum within every possible window does not exceed the detection threshold. We have $(e+k)/{(e+t)}\leq T$, which yields $e\leq (tT-k)/(1-T)$. We also have $e^\prime=0$ since no mask is applied. Therefore, we have the class evidence to satisfy $s_{\bar{y}} = t+e \leq (t-k)/(1-T)$, where $T\in[0,1]$.
\end{enumerate}
Combining the above four cases, we have the upper bound of the target class evidence to be $(t-k)/(1-T)=\textsc{Sum}(\hat{\mathbf{u}}_{\bar{y}}\odot (\mathbf{1}-\mathbf{v}^*_\mathbf{w}))/(1-T)$.

\end{proof}

With Lemma~\ref{lemma2}, we can obtain the upper bound of the wrong class evidence in Algorithm~\ref{alg-provable-masking}. In order to determine the lower bound of true class evidence (Line 7-9), we only need to use a larger mask for the \textsc{Detect} sub-procedure, i.e., modifying Line 8 of Algorithm~\ref{alg-provable-masking} to $\textsc{Detect}(\hat{\mathbf{u}}_y\odot(\mathbf{1}-\mathbf{w}),\mathcal{V},T)$. Given the new bounds for the true class and wrong classes, we can run Algorithm~\ref{alg-provable-masking} to obtain tighter results for the provable accuracy for larger mask window sizes, as shown in Table~\ref{tab-mismatch}.
\begin{table}[t]
    \centering
        \caption{Top-k accuracies of Mask-BN-17 on ImageNet}
    \resizebox{\linewidth}{!}{\scriptsize
    \begin{tabular}{c|c|c|c|c|c|c}
    \toprule
          Patch size &  \multicolumn{2}{c|}{1\% pixels} &\multicolumn{2}{c|}{2\% pixels} & \multicolumn{2}{c}{3\% pixels}  \\
         \midrule
     Accuracy    & clean &robust& clean &robust& clean &robust\\
         \midrule
     Top-1 & 55.1\%&32.2\% &54.6\%&26.0\% & 54.1\% & 19.7\% \\
     Top-2 &65.9\%&48.3\%&65.5\% &43.8\% &64.9\%& 38.2\%\\
     Top-3 &71.3\% & 52.2\% & 70.8\% & 48.7\% & 70.2\%& 44.1\%\\
     Top-4&74.6\% & 53.9\% & 74.2\% & 51.3\% & 73.7\% & 47.4\%\\
     Top-5& 77.0\% & 54.8\% & 76.6\% & 52.9\% & 76.2\% & 49.6\%\\
     \bottomrule
    \end{tabular}}
    \label{tab-topk}
\end{table}
\begin{table}[t]
    \centering
        \caption{Performance of \framework on additional datasets} 
   \resizebox{0.8\linewidth}{!} { \scriptsize
    \begin{tabular}{c|c|c|c|c}
    \toprule
          Dataset &  \multicolumn{2}{c|}{Cats-vs-dogs~\cite{dc}} &\multicolumn{2}{c}{Flowers~\cite{tfflowers}}   \\
         \midrule
     Accuracy    & clean &robust& clean &robust \\
         \midrule
    ResNet-50 &98.4\% & -- & 94.5\%& --\\
     Mask-BN &97.5\% & 91.4\%&91.4\% &83.5\%\\
     Mask-DS &94.8\% &80.2\%&90.5\%&75.6\%\\

     \bottomrule
    \end{tabular}}
 \label{tab-more-ds}
\end{table}

\section{Additional Top-k Analysis}\label{apx-topk}

In this section, we show how to calculate top-k provable robust accuracy and provide additional top-k results for the 1000-class ImageNet dataset.

\noindent \textbf{Top-k provable robustness.} In Algorithm~\ref{alg-provable-masking}, we compare the maximum wrong class evidence $\max_{y^\prime\in\mathcal{Y}^\prime}(\overline{s}_{y^\prime})$ with the lower bound of true class evidence $\underbar{s}_y$ to determine the feasibility of an attack. To determine the top-k provable robustness, we first create a set $\mathcal{S} \gets \{y^\prime\in\mathcal{Y}^\prime | \overline{s}_{y^\prime}>\underbar{s}_y\}$ for wrong classes whose evidence is larger than the lower bound of the true class evidence $\underbar{s}_y$. Next, if the set size $|\mathcal{S}|$ is larger than $k-1$, we assume that a top-k attack is possible; otherwise, the image is robust.

\noindent \textbf{Additional results for ImageNet.} We report the top-k clean accuracy and provable robust accuracy in Table~\ref{tab-topk}. Notably, Mask-BN achieves a 77.0\% clean and 54.8\% provable robust top-5 accuracy against a 1\% pixel patch for the extremely challenging 1000-class classification task.

\section{Defense Performance for Additional Datasets}\label{apx-more-ds}

In this section, we report our provable defense performance for two additional image classification datasets to demonstrate the broad applicability of \framework.

\noindent\textbf{Cats-vs-dogs~\cite{dc}.} The Cat-vs-dog dataset~\cite{dc} contains 25000 images of cats and dogs and is used for binary image classification. We randomly split the dataset and take 80\% images for training and the remaining 20\% for validation. We resize and crop all images 224$\times$224 model inputs, and use the same set of hyper-parameters used for the ImageNette dataset. We report the defense performance for robust masking with BagNet-17 (Mask-BN) and DS-25-ResNet-50 (Mask-DS) against a 2\% pixel patch in Table~\ref{tab-more-ds}. As shown in the table, \framework performs well on the Cats-vs-dogs dataset: the clean accuracy of Mask-BN is only 0.9\% lower compared with ResNet-50, and its provable robust accuracy is as high as 91.4\%.

\noindent\textbf{Flowers~\cite{tfflowers}.} The Flowers dataset~\cite{tfflowers} has 3670 flower images from 5 different categories: \texttt{daisy}, \texttt{dandelion}, \texttt{roses}, \texttt{sunflowers}, \texttt{tulips}. We take a random subset of 80\% images as the training set and the remaining images as the validation set. We use the same hyper-parameters as ImageNette dataset to train the Flowers model. We report the results for a 2\% pixel patch in Table~\ref{tab-more-ds}. \framework also achieves similarly good robustness on the Flowers dataset: Mask-BN achieves a 91.4\% clean accuracy and 83.5\% provable robust accuracy for the 5-class classification task.

\begin{table}[t]
    \centering
        \caption{Invariance of BagNet-17 predictions to feature masking (CIFAR-10)}
  \resizebox{\linewidth}{!}{ \scriptsize \begin{tabular}{c|c|c|c|c|c}
    \toprule
    Window size& 0$\times$0&2$\times$2&4$\times$4 & 6$\times$6 & 8$\times$8 \\
    \midrule
    Masked accuracy &85.4\% & 85.3\% &  85.1\% & 84.7\%&  83.8\% \\
    \% images   & 14.6\% & 18.8\% &22.6\%&27.0\% & 29.0\%\\
    \% windows per image& 0\% & 0.4\% & 1.1\% &2.1\%& 3.7\% \\
    \bottomrule
    \end{tabular}}
    \label{tab-feature-mask-cifar}
\end{table}

\begin{table}[t]
    \centering
        \caption{Effect of logits clipping values on vanilla models (CIFAR-10)}\label{tab-clipping-cifar}
    \resizebox{\linewidth}{!}{\scriptsize \begin{tabular}{c|c|c|c|c|c}
    \toprule
    $(c_l,c_h)$& $(-\infty,\infty)$&$(0,\infty)$&$(0,50)$ & $(0,15)$ & $(0,5)$ \\
    \midrule
    ResNet-50 & 97.0\% & 97.0\% & 97.0\% & 96.7\% & 93.9\%\\
    BagNet-33 &91.3\% & 90.8\%& 89.3\% & 87.1\% & 84.7\%\\
    BagNet-17 & 85.4\% & 85.2\%& 83.2\% & 73.3\% & 64.5\%\\
    BagNet-9 &76.4\% & 75.8\% & 75.1\% & 66.3\% & 57.1\%\\

    \bottomrule
    \end{tabular}}

\end{table}

\section{Additional results for CIFAR-10}\label{apx-cifar}

In Section~\ref{sec-eval-detail}, we provide a detailed analysis of our defense on ImageNette. In this subsection, we provide additional results for CIFAR-10. We will report defense performance against a 2$\times$2 (0.4\% pixels), 5$\times$5 (2.4\%), and 8$\times$8 (6.2\%) adversarial patch. We re-scale the 32$\times$32 images to 192$\times$192 with bicubic interpolation and use BagNet-17 for most of the evaluation. Note that the patch size is also re-scaled proportionally.

\noindent\textbf{Prediction invariance of vanilla models to feature masking.} For BagNet-17 on a re-scaled 192$\times$192 image, we will have 22$\times$22 local features. The results for prediction invariance to feature masking are reported in Table~\ref{tab-feature-mask-cifar}. We can see that the model predictions are generally invariant to small masks. However, we do find that BagNet is more sensitive to feature masking on CIFAR-10 than on ImageNet (Table~\ref{tab-feature-mask}), which also explains its relatively weaker provable robustness performance on CIFAR-10.

\noindent\textbf{Effect of clipping on vanilla models.} We report the effect of clipping in Table~\ref{tab-clipping-cifar}. We also find models trained on CIFAR-10 are more sensitive to clipping than those trained on ImageNet (Table~\ref{tab-clipping}). Our explanation is that most of the small pixel patches of tiny images like CIFAR-10 contain too little visual information and are even unrecognizable for a human. Therefore, most local predictions are incorrect, and the correct local prediction has to use large logits values to dominate the global prediction, which leads to the sensitivity to clipping.

\noindent\textbf{Effect of receptive field sizes on defended models.} We report clean accuracy and provable robust accuracy of our defense for BagNet-33, BagNet-17, and BagNet-9, which have a receptive field of 33$\times$33, 17$\times$17, and 9$\times$9, respectively, against different patch sizes in Table~\ref{tab-masking-provable-cifar}. The observation is similar to results for ImageNette in Table~\ref{tab-masking-provable}: a model with a larger receptive field has better clean accuracy but a larger receptive field results in a larger fraction of corrupted features and thus a larger gap between clean accuracy and provable robust accuracy. 

\begin{table}[t]
    \centering
        \caption{Effect of receptive field sizes on Mask-BN (CIFAR-10)}
    \resizebox{\linewidth}{!}{\scriptsize
    \begin{tabular}{c|c|c|c|c|c|c}
    \toprule
          Patch size &  \multicolumn{2}{c|}{0.4\% pixels} &\multicolumn{2}{c|}{2.4\% pixels} & \multicolumn{2}{c}{6.2\% pixels}  \\
         \midrule
     Accuracy    & clean &robust& clean &robust& clean &robust\\
         \midrule
     Mask-BN-33 & 89.8\% &61.0\% & 89.0\% & 42.7\% & 87.9\% &20.6\%\\
     Mask-BN-17 & 84.5\% &63.8\%&83.9\%& {47.3\%} & 83.3\%& 26.8\%\\
     Mask-BN-9 & 75.7\%  & 62.4\% &75.6\% & 51.7\% &75.0\% & 37.5\%\\
     \bottomrule
    \end{tabular}}
    \label{tab-masking-provable-cifar}
\end{table}

\noindent\textbf{Effect of large masks on defended models.} We report the clean and provable robust accuracy of BagNet-17 when using a conservatively large mask in Table~\ref{tab-mismatch-cifar}. We can have a similar observation as that of ImageNette in Table~\ref{tab-mismatch}.

\begin{table}[t]
    \centering
        \caption{Effect of large masks on Mask-BN-17 (CIFAR-10) }
   \resizebox{\linewidth}{!}  {
    \begin{tabular}{c|c|c|c|c}
    \toprule
     \backslashbox{mask}{patch}     & clean &  0.4\% pixels &2.4\% pixels & 6.2\% pixels  \\
         \midrule
    0.4\% pixels&84.5\% & {63.8\% }& -- & -- \\ 
     2.4\% pixels & 83.9\% &58.9\% &47.3\%& -- \\
     4.4\% pixels &83.5\% &56.2\% & 44.7\%& --\\
     6.2\% pixels & 83.3\%&52.8\% &41.7\%& 26.80\% \\
     \bottomrule
    \end{tabular}}
    \label{tab-mismatch-cifar}
\end{table}

\noindent\textbf{Effect of the detection threshold on defended models.} We study the model performance of BagNet-17 against a 2.4\% pixel patch as we change the detection threshold $T$ from $0.0$ to $1.0$. The results are reported in Table~\ref{tab-masking-thres-cifar} and are similar to those for ImageNette in Table~\ref{tab-masking-thres}.

\begin{table}[t]
    \centering
        \caption{Effect of different detection thresholds on Mask-BN-17 (CIFAR-10)}   
    \resizebox{\linewidth}{!}{\scriptsize \begin{tabular}{c|c|c|c}
    \toprule
          & Clean accuracy & Provable accuracy & Detection FP  \\
         \midrule
     T-0.0 & 84.5\% & 47.3\%& 100\%    \\
     T-0.2 & 80.5\%  &14.8\%  & 31.2\%   \\
     T-0.4 & 85.1\%  &14.3\% & 0.08\%  \\
     T-0.6 &  85.2\%   &0.7\% & 0.02\% \\
     T-0.8 & 85.2\% & 0\%&  0\%   \\
     T-1.0 & 85.2\% & 0\% &0\%   \\
     \bottomrule
    \end{tabular}}
 \label{tab-masking-thres-cifar}
\end{table}

\noindent\textbf{Effect of using different feature types for defended models.} The results for BagNet-17 with different features are reported in Table~\ref{tab-feature-bn-cifar}. Using logits as the feature type has a much better performance than confidence and prediction in terms of clean accuracy and provable accuracy. Results for Mask-DS are shown in Table~\ref{tab-feature-ds-cifar}. Mask-DS works better when we use prediction or confidence as feature types due to its different training objectives.

\begin{table}[t]
    \centering
        \caption{Effect of feature types on Mask-BN (CIFAR-10)}
    \resizebox{\linewidth}{!}{\scriptsize
    \begin{tabular}{c|c|c|c|c|c|c}
    \toprule
          Patch size &  \multicolumn{2}{c|}{0.4\% pixels} &\multicolumn{2}{c|}{2.4\% pixels} & \multicolumn{2}{c}{6.2\% pixels}  \\
         \midrule
      Accuracy   & clean &robust& clean &robust& clean &robust\\
         \midrule
     Logits&  84.5\% &63.8\%&83.9\%& {47.3\%} & 83.3\%& 26.8\%\\
     Confidence&69.6\% & 50.4\% & 69.6\% & 38.0\% & 69.1\% & 24.5\%\\
     Prediction & 65.9\% & 49.6\% & 66.5\% & 34.6\% & 66.4\% & 21.0\%\\

     \bottomrule
    \end{tabular}}
    \label{tab-feature-bn-cifar}

\end{table}

\begin{table}[b]
    \centering
        \caption{Effect of feature types on Mask-DS (CIFAR-10)}
    \resizebox{\linewidth}{!}{\scriptsize
    \begin{tabular}{c|c|c|c|c|c|c}
    \toprule
          Patch size &  \multicolumn{2}{c|}{0.4\% pixels} &\multicolumn{2}{c|}{2.4\% pixels} & \multicolumn{2}{c}{6.2\% pixels}  \\
         \midrule
     Accuracy    & clean &robust& clean &robust& clean &robust\\
         \midrule
     Logits&85.0\%& 60.6\%  & 84.3\% & 39.0\% & 83.4\% & 13.8\%\\
     Confidence & 84.7\% & 69.1\% &  84.6\% & 57.7\% & 84.5\% & 41.8\% \\
     Prediction & 83.9\% & 69.2\% & 83.8\% & 56.8\% & 84.0\%& 41.7\%\\
     \bottomrule
    \end{tabular}}
    \label{tab-feature-ds-cifar}
\end{table}



\section{Detailed Comparison Between Robust Masking and De-randomized Smoothing}\label{apx-generalization}

In Section~\ref{sec-generalization-cbn-ds}, we discussed that \framework is a generalization of DS-ResNet~\cite{levine2020randomized}, and our robust masking outperforms DS-ResNet due to the utilization of spatial information. In this section, we provide more details supporting this claim.

\noindent\textbf{Leveraging spatial information.} We will demonstrate the effect of spatial information in the certification process with a toy example. Assume the original local prediction vector for an image is [\texttt{C},\texttt{C},\texttt{W},\texttt{C},\texttt{C},\texttt{W},\texttt{C},\texttt{C},\texttt{W},\texttt{C},\texttt{C}], where \texttt{C} refers to a correct prediction and \texttt{W} refers to a wrong prediction, and the adversary can additionally corrupt $3$ local predictions. In this example, there are $8$ correct predictions and $3$ wrong predictions in the original prediction vector. Since their difference ($5$) is smaller than two times the upper bound of the number of corrupted predictions ($2\cdot3=6$), DS cannot certify this image. However, the adversary can only corrupt $3$ \emph{consecutive} local predictions, and the optimal attack can result in $6$ correct predictions and 5 wrong predictions, which means this image should be certified. On the other hand, our robust masking utilizes the spatial information and can certify this image, leading to a better provable robustness.

\section{Additional Discussion on Different Patch Shapes}\label{apx-patch-shape}

In our evaluation, we only report our defense performance against a square patch. In this section, we report additional evaluation results for different patch shapes. We note that Equation~\ref{eq-receptive} applies to both x-axis and y-axis of the image. Therefore, given a patch shape, we only need to compute the mask window size for each axis to achieve robustness. In Table~\ref{tab-shape}, we report the performance of Mask-BN-17 on different rectangular/square patches on ImageNette and ImageNet. As shown in the table, our defense for rectangular patches works as well as that for square patches; the provable robust accuracy is largely correlated with the number of patch pixels instead of the patch shape. However, we note that our robust masking defense requires a prior estimation of the patch shapes. We leave how to design a parameter-free secure aggregation as future work.

\begin{table}[]
    \centering
     \resizebox{0.9\linewidth}{!}
     {\scriptsize
    \begin{tabular}{c|c|c|c|c}
    \toprule
     Dataset    & \multicolumn{2}{c|}{ImageNette}&\multicolumn{2}{c}{ImageNet} \\
         \midrule
         Accuracy& clean & robust &clean &robust\\
         \midrule
    16$\times$144 (2304px) &95.0\%&78.0\%&54.1\%&11.7\%\\
    24$\times$88 (2112px)& 94.8\% & 79.8\%&53.6\%&14.2\%\\
32$\times$72 (2304px) &94.6\% & 78.9\%&53.4\%&13.4\%\\
40$\times$56 (2240px) &94.7\% & 79.6\%&53.5\%&14.4\%\\
\underline{48$\times$48 (2304px)} &\underline{94.6\%}&\underline{79.4\%}&\underline{53.5\%}&\underline{14.2\%}\\
56$\times$40 (2240px) &94.6\%&80.1\%&53.5\%&14.7\%\\
72$\times$32 (2304px)&94.8\%&79.7\%&53.6\%&14.2\%\\
88$\times$24 (2112px)& 94.8\% & 80.7\%&53.9\%&15.3\%\\
144$\times$16 (2304px) &94.9\% & 80.4\%&54.2\%&13.7\%\\
        \bottomrule
    \end{tabular}}
    \caption{Mask-BN-17 performance on different patch shapes}
    \label{tab-shape}
\end{table}


\section{Additional Discussion on Multiple Patches}\label{apx-multiple-patch}

In this paper, we focus on the threat of the adversary arbitrarily corrupting \emph{one} contiguous region. In this section, we discuss how \framework can deal with multiple adversarial patches.

\noindent\textbf{Merge multiple patches into one large patch.} The most straightforward way to approach multiple patches is to consider a larger contiguous region that contains all patches. The analysis in the paper can be directly applied. However, when the multiple patches are far away from each other, the merged single region can be too large to have decent model robustness. When this is the case, we have the following alternatives.\footnote{We note that when patches are far away from each other, the problem is closer to the global adversarial example with a $L_0$ constraint, which is orthogonal to our work.}

\noindent\textbf{Mask individual local features.} Our robust masking presented in Section~\ref{sec-masking} masks the feature \emph{window} with the highest class evidence. As an alternative, we can mask $\alpha$ \emph{individual} local features with top-$\alpha$ highest class evidence. Such individual feature masking is agnostic to the number of patches, and we can easily re-prove Lemma~\ref{lemma} to have the same upper bound of wrong class evidence. However, the lower bound of true class in this masking mechanism is reduced compared with window masking, and this will lead to a drop in provable robust accuracy. 

\noindent \textbf{Use alternative secure aggregation.} As discussed in Section~\ref{sec-future}, a promising direction of future work is to explore parameter-free secure aggregation mechanisms. We note that we have already seen concrete examples of alternative aggregation that can deal with multiple patches in Section~\ref{sec-generalization-cbn-ds}, where we discuss how to reduce \framework to CBN and DS.

\begin{figure}[t]
    \centering
    \includegraphics[width=\linewidth]{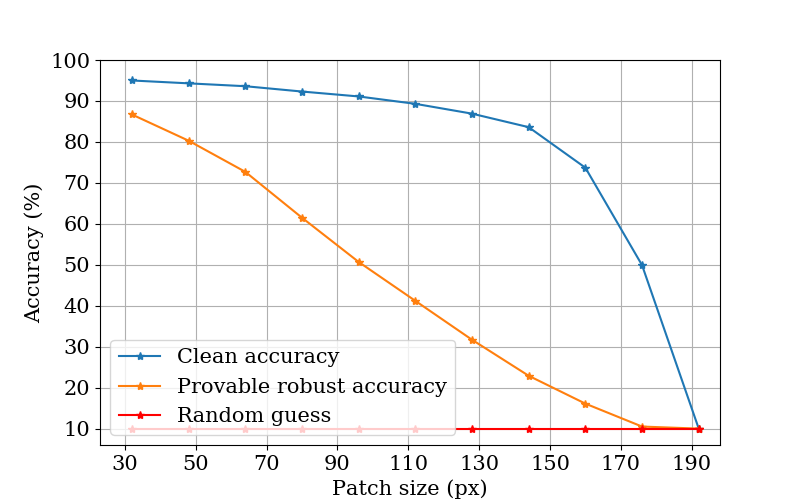}
    \caption{Performance of Mask-BN-17 on ImageNette against various patch sizes.}
    \label{fig-large-patch}
\end{figure}

\begin{figure}[t]
    \centering
    \includegraphics[width=\linewidth]{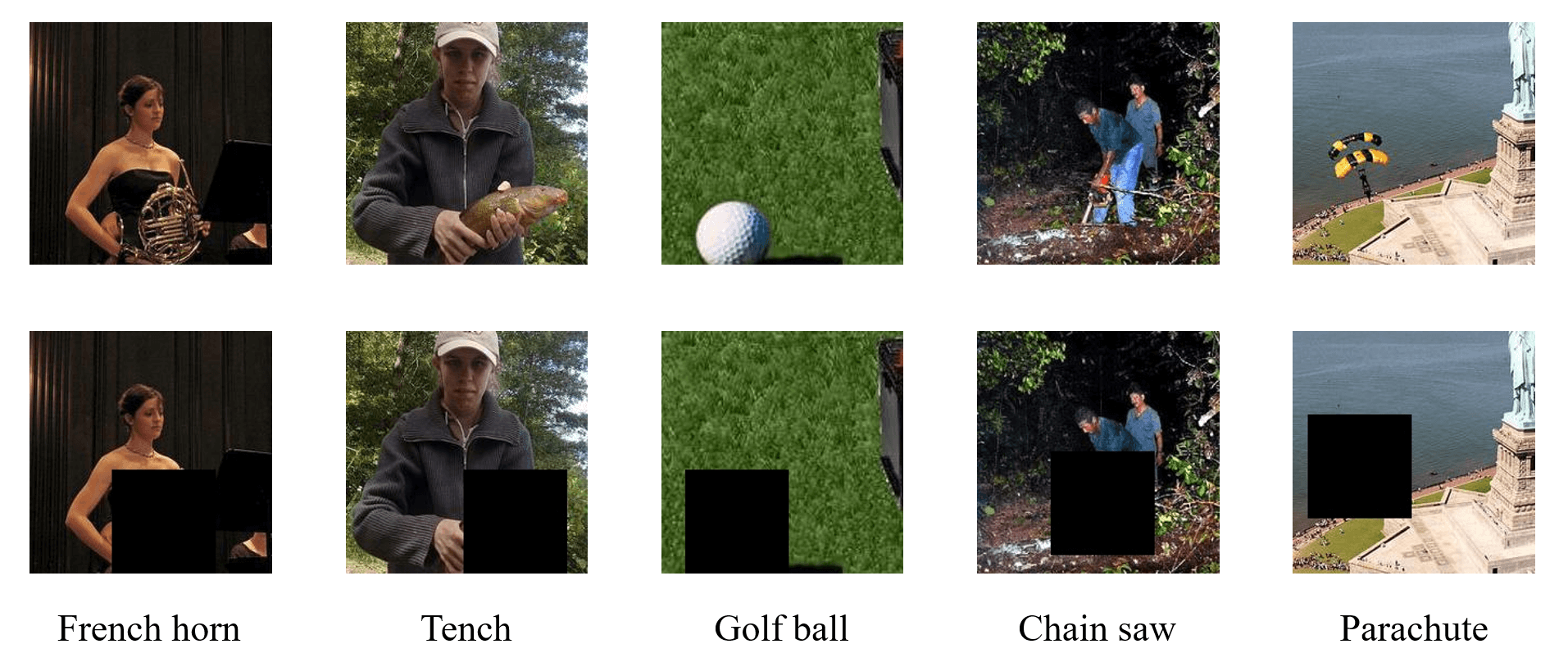}
    \vspace{-1.5em}
    \caption{Visualization of large occlusion with a 96$\times$96 pixel block on the 224$\times$224 image.}
    \label{fig-large-patch-visu}
\end{figure}

\section{Additional Discussion on the Limits of \framework}\label{apx-large-patch}

In Section~\ref{sec-evaluation}, we evaluate our defense against 1-3\% pixel patches. In this section, we take Mask-BN-17 on ImageNette for a case study to analyze the defense performance when facing a much larger patch. We report the performance of Mask-BN-17 against various patch sizes in Figure~\ref{fig-large-patch}. Note that the image is in the shape of 224$\times$224; a 32$\times$32 square patch takes up 2\% pixels. As shown in the figure, the clean accuracy of Mask-BN-17 drops slowly as the patch size increases. When the patch size is as large as 192$\times$192 pixels, the patch will appear in the receptive field of all local features of BagNet-17 and our defense reduces to a random guess (10\% accuracy for the 10-class classification task). Similarly, the provable robust accuracy drops as the patch becomes larger. We note that this drop also results from the limitation of classification problem itself when the patch is large. In Figure~\ref{fig-large-patch-visu}, we visualize five images and occlude them with a 96$\times$96 pixel block. As shown in the figure, a large pixel block covers the entire salient object and makes the classification almost impossible (recall that our threat model allows the adversary to put a patch at any location of the image). Notably, our defense still achieves a 91.1\% clean accuracy and 50.7\% provable robust accuracy against this large 96$\times$96 patch. We believe this analysis further demonstrates the strength of our defense.

\end{document}